\documentclass[letterpaper]{article}
\usepackage{uai2020}
\usepackage{times}
\usepackage[margin=1in]{geometry}
\usepackage{times}
\usepackage{mathrsfs}
\usepackage{amsfonts}
\usepackage[utf8]{inputenc} 
\usepackage[T1]{fontenc}
\usepackage{url}
\usepackage{ifthen}
\usepackage{cite}
\usepackage[cmex10]{amsmath}

\usepackage{bm}
\usepackage{amsfonts}
\usepackage{tikz}
\usetikzlibrary{arrows}
\usepackage{subfigure}
\usepackage{graphicx,booktabs,multirow}
\usepackage{epstopdf}
\usepackage{subfigure}
\usepackage{multirow}
\usepackage{tabularx} 
\usepackage{booktabs}
\usepackage{algorithm}
\usepackage{algorithmicx}
\usepackage{algpseudocode} 

\usepackage{amsthm}

\newtheorem{thm}{Theorem}[section]
\newtheorem{lem}{Lemma}[section]
\newtheorem{prop}{Proposition}[section]

\theoremstyle{definition}

\theoremstyle{remark}
\newtheorem*{rem}{Remark}

\title{Complete Dictionary Learning via $\ell_p$-norm Maximization}

\author{Yifei Shen$^{1}$\thanks{\quad These two authors contributed equally.}, Ye Xue$^{1}$\footnotemark[1], Jun Zhang$^{2}$, Khaled B. Letaief$^{1}$, Vincent Lau$^{1}$\\
	{\small $^1$Hong Kong University of Science and Technology, Hong Kong, China,}
	\emph{\small \{yshenaw, ye.xue,eekhaled,eeknlau\}@ust.hk}\\
	{\small $^2$The Hong Kong Polytechnic University, Hong Kong, China,}
	\emph{\small jun-eie.zhang@polyu.edu.hk}
}

\begin{document}

\maketitle

\begin{abstract}
Dictionary learning is a classic representation learning method that has been widely applied in signal processing and data analytics. In this paper, we investigate a family of $\ell_p$-norm ($p>2,p \in \mathbb{N}$) maximization approaches for the complete dictionary learning problem from theoretical and algorithmic aspects. Specifically, we prove that the global maximizers of these formulations are very close to the true dictionary with high probability, even when Gaussian noise is present. Based on the generalized power method (GPM), an efficient algorithm is then developed for the $\ell_p$-based formulations. We further show the efficacy of the developed algorithm: for the population GPM algorithm over the sphere constraint, it first quickly enters the neighborhood of a global maximizer, and then converges linearly in this region. Extensive experiments demonstrate that the $\ell_p$-based approaches enjoy a higher computational efficiency and better robustness than conventional approaches and $p=3$ performs the best.
\end{abstract}

\section{INTRODUCTION}
\emph{Dictionary learning} is a classic unsupervised representation learning method \cite{olshausen1996emergence}. Given data $\bm{Y}$, it identifies a representation basis $\bm{D}_0$ and the corresponding coefficients $\bm{X}_0$ such that $\bm{Y} \approx \bm{D}_0\bm{X}_0$ with $\bm{X}_0$ being sufficiently sparse. Given its powerful capability of exploiting low-dimensional structures in high-dimensional data, dictionary learning has found wide applications in signal and image processing \cite{elad2010high}.

Solving the dictionary learning problem inevitably involves non-convex optimization \cite{sun2015nonconvex}, and is thus highly challenging. A key ingredient underlying the recent advancement in this area is innovative mathematical formulations. A natural formulation is to minimize the $\ell_0$ norm of $\bm{X}_0$ for inducing sparsity. However, $\ell_0$ minimization with a non-convex constraint is computationally intractable. Thus, surrogate objective functions are explored, which ideally should come with theoretical guarantees for recovering the dictionary, as well as efficient algorithms. In particular, $\ell_1$-norm minimization based formulations, which promote sparsity in problems such as compressive sensing, have been widely adopted in dictionary learning \cite{sun2015complete,bai2018subgradient,gilboa2018efficient,wang2019unique}. While such formulations enjoy theoretical guarantees, they face computational challenges for high-dimensional data. Particularly, existing algorithms can only deal with one row of a dictionary at a time. Given the high computational complexity of classic algorithms for recovering one row, e.g., the Riemannian trust region algorithm or subgradient descent, repeatedly solving the problem for $n$ times to recover the whole dictionary leads to prohibitive complexity. Additionally, $\ell_1$-based methods are known to be sensitive to noise in the observation \cite{wang2019unique,zhai2019understanding}. Thus, formulations that lead to more efficient and robust methods are needed.

Recently, a novel $\ell_4$-norm maximization formulation was proposed in \cite{zhai2019complete,zhai2019understanding}, which is able to recover the entire dictionary at once. It was shown that the global maximizers of the $\ell_4$-based formulation are very close to the true dictionary. Moreover, the concaveness of the formulation enables a fast fixed-point type algorithm, named \emph{matching, stretching, and projection} (MSP), which achieves hundreds of times speedup compared with existing methods. This new formulation is motivated by the fact that maximizing $\ell_{2k+2}$-norm promotes spikiness and sparsity \cite{li2018global,zhang2019structured}. In experiments, $\ell_4$-norm was found to be the best among all the $\ell_{2k+2}$-norms in terms of the sample complexity and computational efficiency.

The essence of the $\ell_{2k+2}$-norm approach is to maximize an $\ell_{2k+2}$-norm over an $\ell_2$-norm constraint. In principle, maximizing \emph{any} higher-order norm over a lower-order norm constraint leads to sparse and spiky solutions. Thus, we conjecture that the dictionary can be recovered via maximizing \emph{any} $\ell_p$-norm ($p>2$), which is tested numerically in Fig. \ref{fig:l3-6}. It is demonstrated that the dictionary is recovered with high probability for $p=3, 4, 5, 6$. Particularly, the $\ell_3$-based formulation enjoys the lowest sample complexity\footnote{Sample complexity here means the minimal required number of samples to successfully recover the true dictionary.}. These observations lead to the following intriguing questions:
\begin{itemize}
	\item Can all the $\ell_p$-norm maximization based formulations provably recover the true dictionary?
	\item Which $p$ should we pick for practical applications?	
	\item What advantages do they enjoy over the traditional approaches, e.g., $\ell_1$-based approaches?
\end{itemize}

Unfortunately, the analysis in \cite{zhai2019complete} cannot be extended to $\ell_{2k+1}$-norms, and thus could not address the above questions. In this paper, we endeavor to develop more general theoretical results for $\ell_p$-norm ($p>2, p \in \mathbb{N}$) based formulations\footnote{Our analysis can be extended to $2<p<\infty$ with minor modification. However, due to the high computational complexity of taking fractional power, we will not discuss $p \notin \mathbb{N}$ in this paper.}, which will lead to a better understanding of such methods and also provide guidelines to find more efficient algorithms.

\begin{figure}[htb]
	\centering
	\includegraphics[width=0.45\textwidth]{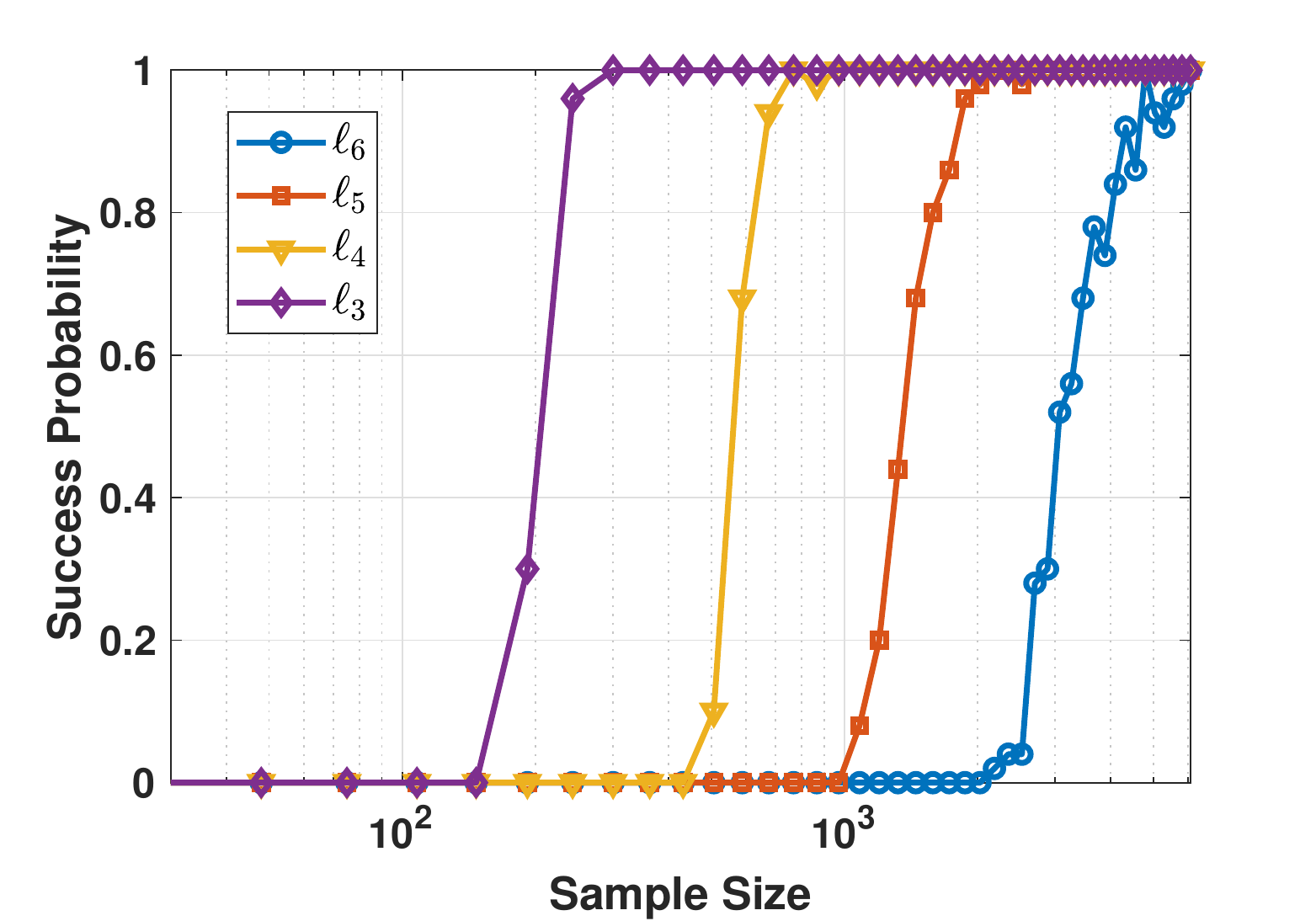}
	\caption{The successful dictionary recover probability with different values of $p$ when $n=30$ and $\theta=0.3$.}
	\label{fig:l3-6}
\end{figure}

\subsection{Contributions}
In this paper, we study the $\ell_p$-norm ($p>2, p \in \mathbb{N}$) maximization based dictionary learning. 
\begin{itemize}
	\item We prove that as long as the number of samples is larger than $\Omega\left(nk^{\frac{p}{2}}\log^{\frac{p}{2}+1}(n)\right)$, where $n$ is the size of the dictionary and $k$ is the sparsity level, the global maximizers of all $\ell_p$-norm maximization based formulations are very close to the true dictionary with high probability. The dictionary can be recovered even in the presence of the Gaussian noise.
	\item An efficient algorithm is developed based on the generalized power method \cite{journee2010generalized}, which applies to $p>2$. It is proved that the population generalized power method enjoys a desirable global convergence behavior over the sphere constraint. Specifically, the convergence involves two stages, where the first stage only takes a few iterations and the second stage enjoys a linear convergence rate.
	\item To guide the practical application, we prove that the $\ell_3$-based approach enjoys the lowest sample complexity and is the most robust among all the $\ell_p$-norm maximization based formulations. The experiments will further demonstrate that the $\ell_3$-based approach is also more time-efficient and robust than existing methods, including K-SVD, $\ell_1$ and $\ell_4$-based approaches.
\end{itemize}

\subsection{Notations and Terminologies}
\paragraph{Asymptotic notations:} Throughout the paper, $f(n) = \mathcal{O}(g(n))$ means that there exists a constant $c>0$ such that $f(n) \leq c|g(n)|$; $f(n) = \Theta(g(n))$ means that there exists constants $c_1,c_2$ such that $c_1|g(n)|\leq f(n) \leq c_2 |g(n)|$; $f(n) = \Omega(g(n))$ means that there exists constants $c_3 > 0$ such that $c_3|g(n)|\leq f(n)$

\paragraph{Norms:} $\|\cdot\|_p$ is the element-wise $p$-th norm of a vector or matrix, i.e., $\|\bm{A}\|_p = (\sum_{i}\sum_{j} |A_{ij}|^p)^{\frac{1}{p}}$. Likewise, $\|\cdot\|_F$ and $\|\cdot\|$ denote the Frobenius norm and operator norm of a matrix, respectively.

\paragraph{Stiefel manifold:}
The Stiefel manifold $\text{St}(n,m)$ is defined as the subspace of orthonormal N-frames in $\mathbb{R}^n$, namely,
\begin{equation}
\text{St}(n,m) = \{\bm{\Gamma} \in \mathbb{R}^{n\times m}: \bm{\Gamma}^*\bm{\Gamma} = \bm{I}_m  \}
\end{equation}
where $\bm{I}_m$ is the $m \times m$ identity matrix. We further denote the orthogonal group as $\mathbb{O}(n) = \text{St}(n,n)$

\paragraph{Distributions:} We denote a Bernoulli distribution with $\theta$ non-zero probability as $\text{Ber}(\theta)$. The Bernoulli-Gaussian distribution is denoted by $\mathcal{BG}(\theta)$ and defined as $x=b\cdot g$, where $b \sim \text{Ber}(\theta)$ and $g \sim \mathcal{N}(0,1)$.

\paragraph{Sign-permutation ambiguity:} Note that for a sparse matrix $\bm{X}_0$ and any signed permutation matrix $\bm{P}$, $\bm{X}_0$ and $\bm{X}_0\bm{P}$ are equally sparse. Thus, we consider that a dictionary $\bm{D}_0$ is successfully recovered if we find any signed permutation of $\bm{D}_0$.

\section{$\ell_p$-NORM MAXIMIZATION BASED COMPLETE DICTIONARY LEARNING}
In this section, we study the statistical performance of $\ell_p$-based formulations of complete dictionary learning and investigate their robustness.
\subsection{Problem Formulation}
We consider an orthogonal (complete) dictionary learning problem with a Bernoulli-Gaussian model, with the following three justifications. First, it has been demonstrated that the performance of complete bases is competitive to over-complete dictionaries in real applications \cite{bao2013fast}. Second, the complete dictionary learning problem can be converted into the orthogonal case through a simple preconditioning \cite{sun2015complete}. Third, the Bernoulli-Gaussian model is a reasonable model for generic sparse coefficients \cite{spielman2012exact,sun2015complete,bai2018subgradient,zhai2019complete,zhai2019understanding}.

Specifically, we assume that each sample $\bm{y}_i \in \mathbb{R}^n$ is generated from a sparse superposition of an orthogonal dictionary $\bm{D}_0 \in \mathbb{O}(n)$, i.e., $\bm{y}_i=\bm{D}_0\bm{x}_i$, where each element in $\bm{x}_i \in \mathbb{R}^n$ is i.i.d. Bernoulli-Gaussian, i.e., $x_{i,j} \sim \mathcal{BG}(\theta)$. Denote $\bm{Y}=\{\bm{y}_1,\bm{y}_2,\cdots,\bm{y}_r\}$, $\bm{X}_0=\{\bm{x}_1,\bm{x}_2,\cdots,\bm{x}_r\}$ and thus $\bm{Y}=\bm{D}_0\bm{X}_0$. Our goal is to simultaneously recover the dictionary $\bm{D}_0$ and coefficients $\bm{X}_0$ from the observation $\bm{Y}$. A good estimate of $\bm{D}_0$, denoted as $\bm{D}$, should maximize the sparsity of the associated coefficients $\bm{X}$. Therefore, a natural $\ell_0$-based formulation of the orthogonal dictionary learning problem is
\begin{equation} \label{eq:l0_st}
\begin{aligned}
&\underset{\bm{D},\bm{X}}{\text{minimize}}
&& \|\bm{X}\|_0 \\
& \text{ subject to } 
&&\bm{Y}=\bm{D}\bm{X}, \bm{D} \in \mathbb{O}(n).
\end{aligned}
\end{equation}
As $\bm{D} \in \mathbb{O}(n)$, we can write $\bm{X} = \bm{D}^*\bm{Y}$. Denote $\bm{D}=[\bm{d}_1,\cdots,\bm{d}_n]$ and $\bm{A} = \bm{D}^*$, and then the formulation (\ref{eq:l0_st}) can be transformed into 
\begin{equation} \label{eq:l0_st2}
\begin{aligned}
&\underset{\bm{A}}{\text{minimize}}
& & \|\bm{A}\bm{Y}\|_0 \\
& \text{ subject to }
& & \bm{A} \in \mathbb{O}(n).
\end{aligned}
\end{equation}

Nevertheless, Problem (\ref{eq:l0_st2}) is difficult to solve due to the combinatorial nature of $\|\cdot \|_0$ and the non-convex constraint. To motivate $\ell_p$-based formulations, we first consider a simpler case of (\ref{eq:l0_st2}), i.e., solving one column of the dictionary by
\begin{equation} \label{eq:l0_s}
\begin{aligned}
&\underset{\bm{a}}{\text{minimize}}
& & \|\bm{a}^*\bm{Y}\|_0
& \text{ subject to }
& & \|\bm{a}\|_2 = 1.
\end{aligned}
\end{equation}

The essence of the existing heuristic formulations, either the $\ell_1$ or $\ell_{4}$ based one, to solve (\ref{eq:l0_s}) is to promote sparsity over an $\ell_2$ constraint. Note that the landscapes of higher-order norms are sharper than lower-order norms. Thus, maximizing \emph{any} higher-order norm over a lower-order norm constraint pushes the variables to extreme values, i.e., $0$ or the maximal, which leads to sparse solutions. Fig. \ref{fig:l4} illustrates this phenomenon on $\mathbb{S}^2$. Therefore, heuristically, we can adopt \emph{any} $\ell_p$-based ($p>2$) formulation for dictionary learning. Specifically, we expect one column of the dictionary to be recovered by solving the following problem

\begin{equation} \label{eq:lp}
\begin{aligned}
&\underset{\bm{a}}{\text{maximize}}
& & \|\bm{a}^*\bm{Y}\|_p^p = \|\bm{a}^*\bm{D}_0\bm{X}_0\|_p^p \\
&\text{ subject to } 
& &\|\bm{a}\|_2=1,
\end{aligned}
\end{equation}
where $p \in \mathbb{N}$ and $p>2$. 

Similarly, $m$ columns of $\bm{D}_0$ can be recovered at once by considering the following optimization problem

\begin{equation} \label{eq:lp_st}
\begin{aligned}
\mathscr{P}:\quad & \underset{\bm{A}}{\text{maximize}}
& & \|\bm{A}\bm{Y}\|_p^p = \|\bm{A}\bm{D}_0\bm{X}_0\|_p^p \\
& \text{ subject to } 
& &\bm{A}^* \in \text{St}(n,m).
\end{aligned}
\end{equation}

The whole dictionary can be recovered at once if $m=n$.

\begin{figure}[htb]
	\centering
	\includegraphics[width=0.4\textwidth]{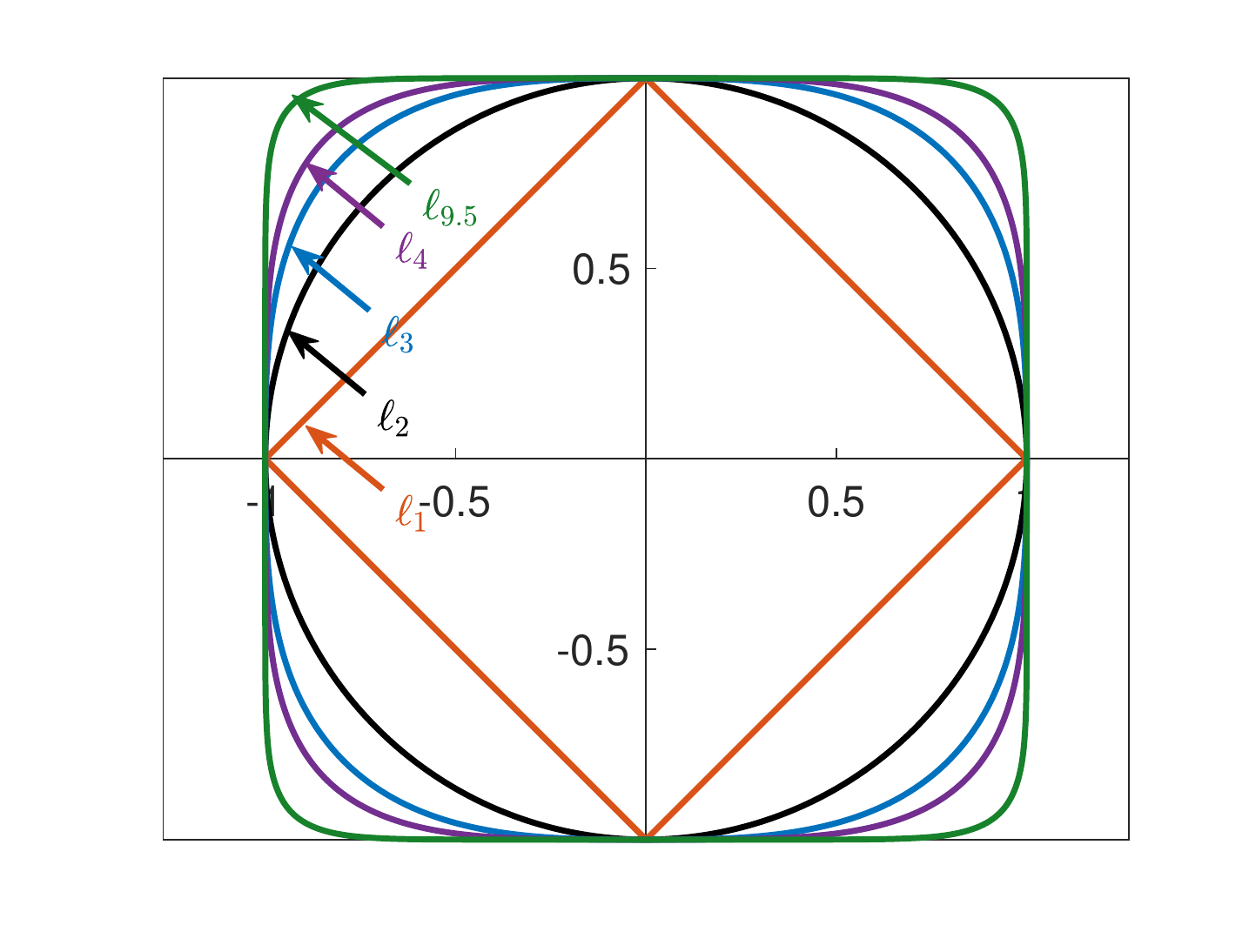}
	\caption{The figure of $\ell_p$-norms. Maximizing any $\ell_p$-norm ($p>2$) over the sphere leads to the solutions on the coordinates, i.e., sparse solutions.}
	\label{fig:l4}
\end{figure}

\subsection{Statistical justification}
In this subsection, we offer statistical justification to our $\ell_p$-based formulation in (\ref{eq:lp_st}). We consider $m=n$ for simplicity in this subsection and $m < n$ can be derived in the same way with minor modifications. We prove that the global maximizers of all $\ell_p$-based formulations are very close to the true dictionary, which is formally stated as below.

\begin{thm} \label{thm:ex_re}
	Let $\bm{X} \in \mathbb{R}^{n \times r}, x_{i,j} \sim \mathcal{BG}(\theta)$ with $\theta \in (0,1)$, $\bm{D}_0 \in \mathbb{O}(n)$ be an orthogonal dictionary, and $\bm{Y} = \bm{D}_0\bm{X}$. Suppose $\hat{\bm{A}}$ is a global maximizer to 
	\begin{equation} 
	\begin{aligned}
	&\underset{\bm{A}}{\text{maximize}}
	& & \|\bm{A}\bm{Y}\|_p^p & & &\text{ subject to } \bm{A} \in \mathbb{O}(n). \notag
	\end{aligned}
	\end{equation}
	Provided that the sample size $r = \Omega\left( \delta^{-2} n\log (n/\delta) (\theta n \log^2 n)^{\frac{p}{2}}\right)$, then for $\delta > 0$, there exists a signed permutation $\bm{\Pi}$, such that
	\begin{align*}
		\frac{1}{n}\left\|\hat{\bm{A}}^* - \bm{D}_0\bm{\Pi} \right\|_F^2 \leq C_{\theta} \delta
	\end{align*}
	with probability at least $1-r^{-1}$ and $C_{\theta}$ is a constant that depends on $\theta$.
\end{thm}

\begin{rem}
	The derivation for the correctness of global maximizers of the $\ell_4$-based formulation in \cite{zhai2019complete} requires a closed-form expression of the expected objective function, which cannot be obtained for $\ell_{2k+1}$-based formulations. In order to develop general theoretical results, we first show that if a formulation satisfies \emph{concentration} and \emph{sharpness conditions} defined in Theorem \ref{thm:cor_g}, its global maximizers are very close to the true dictionary. Then we prove that all $\ell_p$-based ($p>2$) formulations satisfy these two conditions. Our result is consistent with the result in \cite{qu2019geometric} when $p=4$. Please refer to Section \ref{sec:noiseless} for a detailed proof.
\end{rem}

We next present two lemmas to give a better understanding of Theorem \ref{thm:ex_re}. First, Lemma \ref{lem:c_lp} shows that the global maximizers of the population objective are the true dictionary. Second, we figure out how many samples are needed for the concentration of the empirical objective around the population objective in Lemma \ref{lem:c_obj}. 

We first show that the global maximizers are the true dictionary in expectation. 

\begin{lem} (Correctness in expectation) \label{lem:c_lp}
	Denote $\mathcal{D}$ as the set of global maximizers to
	\begin{equation} \label{eq:lp_st_e}
	\begin{aligned}
	&\underset{\bm{A}}{\text{maximize}}
	& & \mathbb{E}_{\bm{Y}}\|\bm{A}\bm{Y}\|_p^p & & &\text{ subject to } \bm{A} \in \mathbb{O}(n), \notag
	\end{aligned}
	\end{equation}
	then $\mathcal{D}=\{\bm{D}_0^* \bm{\Pi}^*| \bm{\Pi} \in \text{SP}(n) \}$, where $\text{SP}(n)$ denotes the group of the signed permutation matrices.  
\end{lem}

Lemma \ref{lem:c_lp} states the correctness of $\ell_p$-based formulations, namely, the global maximizers of the population objective are the true dictionary up to some signed permutations. Due to the law of large numbers, the gap between the empirical objective and expectation objective vanishes as the number of samples goes to infinity. Nevertheless, the sample complexity is an important consideration in dictionary learning. We hope to learn the true dictionary from as few samples as possible. The next proposition states the finite sample concentration, namely, the empirical objective concentrates on the expectation objective as long as the number of samples is $\tilde{\Omega}(nk^{\frac{p}{2}})$, where $k$ is the number of non-zero elements.

\begin{lem} \label{lem:c_obj}
	(Concentration bound of the objective) Suppose $\bm{X} \in \mathbb{R}^{n \times r}$ follows $\mathcal{BG}(\theta)$. For any given $\theta \in (0,1)$ and $\delta > 0$, whenever
	$$r \geq C \delta^{-2} n\log (n/\delta) (\theta n \log^2 n)^{\frac{p}{2}},$$
	we have 
	$$\underset{\bm{A} \in \mathbb{O}(n) }{\sup} \frac{1}{nr}\left| \|\bm{A}\bm{Y}\|_p^p - \mathbb{E}(\|\bm{A}\bm{Y}\|_p^p) \right| \leq \delta$$
	with probability at least $1-r^{-1}$.
\end{lem}

\begin{rem}
	We provide experimental validations for the phase transitions of $\ell_p$-based formulations in Fig. \ref{fig:pt_hm}.
\end{rem}

From Lemma \ref{lem:c_obj}, we see that more samples are required as $p$ increases, and thus a smaller $p$ leads to a lower sample complexity. Fig. \ref{fig:l3-6} illustrates the successful recovery probability versus the number of samples for different values of $p$. It also confirms the exponential increase in the sample complexity as $p$ increases.

\subsection{Robustness}
In practice, the observations are usually noisy. Intuitively, the regions around the global maximizers of the $\ell_p$-norms ($p>2$) are flatter than that of the $\ell_1$-norm, which leads to a higher tolerance to noise. In this subsection, we investigate the robustness of the $\ell_p$-based formulations to Gaussian noise. Other typical noise will be tested in Section \ref{sec:exp} via experiments. 

We consider noisy measurements $\bm{Y}_{N} = \bm{Y} + \bm{G}$ where $\bm{G} \in \mathbb{R}^{n \times r}$ with $G_{i,j} \sim \mathcal{N}(0,\eta^2)$. We find that when the number of samples goes to infinity, the Gaussian noise will not change the global maximizers. However, the noise makes it harder for the objective value to concentrate and thus more samples are needed compared to the clean objective. The next theorem shows that the global maximizers are very close to the true dictionary under Gaussian noise.
\begin{thm} \label{thm:c_obj_n}
	Let $\bm{X} \in \mathbb{R}^{n \times r}, x_{i,j} \sim \mathcal{BG}(\theta)$, $\bm{D}_0 \in \mathbb{O}(n)$ be an orthogonal dictionary, and $\bm{Y}_N = \bm{D}_0\bm{X} + \bm{G}$, $\bm{G} \in \mathbb{R}^{n \times r}$ with $G_{i,j} \sim \mathcal{N}(0,\eta^2)$. Suppose $\hat{\bm{A}}$ is a global maximizer to 
	\begin{equation} 
	\begin{aligned}
	&\underset{\bm{A}}{\text{maximize}}
	& & \|\bm{A}\bm{Y}_N\|_p^p & & &\text{ subject to } \bm{A} \in \mathbb{O}(n) \notag
	\end{aligned}
	\end{equation}
	then for $\delta > 0$, there exists a signed permutation $\bm{\Pi}$, such that
	\begin{align*}
	\frac{1}{n}\left\|\hat{\bm{A}}^* - \bm{D}_0\bm{\Pi} \right\|_F^2 \leq C_{\theta} \delta
	\end{align*}
	with probability at least $1-r^{-1}$ as long as $r=\Omega(\delta^{-2}n\log (n/\delta)((1+\eta^2) n \log n)^{\frac{p}{2}}\xi_\eta^2)$, where $\xi_\eta = (1+\eta^2)^{p/2} + \eta^p - 2(0.5 + \eta^2)^{p/2}$, and $C_{\theta}$ is a constant that depends on $\theta$.
\end{thm}
\begin{rem}
	Please refer to Section \ref{sec:noisy} for a detailed proof.
\end{rem}

Theorem \ref{thm:c_obj_n} shows that as the number of samples is sufficiently large, the global maximizers are very close to the true dictionary. It also suggests that the sample complexity increases as $p$ becomes larger, i.e., a smaller $p$ is more robust to Gaussian noise.

\section{AN EFFICIENT ALGORITHM}
In this section, we develop an efficient algorithm for $\ell_p$-based dictionary learning, and investigate its convergence property. Particularly, an interesting two-stage convergence behavior is revealed and explained.
\subsection{Algorithm for $\ell_p$-based dictionary learning}
We develop our algorithm based on the generalized power method (GPM) algorithm \cite{journee2010generalized}, which is a general optimization method to deal with concave objective functions. From the GPM algorithm, we can derive efficient algorithms for many applications, e.g., \emph{subspace iteration} for finding $k$-largest eigenvectors \cite{bathe2013subspace}, the \emph{matching, stretching, and projection} algorithm for $\ell_4$-based dictionary learning \cite{zhai2019complete}, and efficient algorithms for sparse principle component analysis \cite{journee2010generalized}. These algorithms have been shown to enjoy linear or super-linear convergence rates as well as cheap per-iteration cost (the same cost as the gradient method) in experiments, which motivates us to apply GPM to $\ell_p$-based dictionary learning.

The GPM algorithm aims at maximizing a convex function $f(\cdot)$ over a compact constraint $Q$.

\begin{equation}
\underset{x \in Q}{\text{max} } \quad \underbrace{f(x)}_{\text{Convex}}\notag
\end{equation}

In each iteration, the GPM algorithm maximizes a linear surrogate of the objective function. The procedure of the GPM algorithm is shown in Algorithm \ref{alg:gpm}, where $f' \in \partial f$ is any subgradient.
\begin{algorithm}  
	\caption{Generalized power method \cite{journee2010generalized}}  
	\label{alg:gpm}  
	\begin{algorithmic}[1]
		\State Initialize $\bm{x}^{(0)} \in D$.
		\For{$t=0...T$}
		\State $\bm{x}^{(t+1)} = \underset{\bm{s} \in D}{\text{argmax}} \langle \bm{s}, f'(\bm{x}^{(t)})\rangle,$
		\EndFor
	\end{algorithmic}  
\end{algorithm}

We develop an algorithm for $\ell_p$-based dictionary learning based on the GPM algorithm. Let $f(\bm{A}) = \|\bm{A}\bm{Y}\|_p^p$ and thus $\nabla f(\bm{A}) = \left( |(\bm{A}\bm{Y})^{\circ (p-1)}| \circ \text{sign}(\bm{A}\bm{Y}) \right)\bm{Y}^*$. The update for the GPM algorithm is 

\begin{equation}\label{eq:gpm_lp}
\bm{A}^{(t+1)} = \underset{\bm{s}^* \in \text{St}(n,m)}{\text{argmax}} \langle \bm{s},  \nabla f(\bm{A}^{(t)})\rangle.
\end{equation}

The only thing left is to compute the maximizer over the Stiefel manifold, which is stated in the next lemma.

\begin{lem} \cite{journee2010generalized}
	Let $\bm{C} \in \mathbb{R}^{m \times n}$ with $m \leq n$, and the singular values of $\bm{C}$ is denoted by $\sigma_i(\bm{C}), i=1,\cdots,m$. Then,
	\begin{align*}
		\underset{\bm{s} \in \text{St}(n,m)}{\max} \langle \bm{s},  \bm{C} \rangle = \sum_{i=1}^n \sigma_i(\bm{C})
	\end{align*}
	with maximizer $\bm{s} = \text{Polar}(\bm{C})$. $\text{Polar}(\bm{C})$ denotes the $\bm{U}$ factor of the polar decomposition of the matrix $\bm{C}$ such that
	\begin{align*}
		&\bm{C}=\bm{UP}, && \bm{U} \in \text{St}(n,m), && \bm{P} \in S^m_+.
	\end{align*}
	where $S^m_+$ denotes positive semi-definite matrices. 
\end{lem}

The whole algorithm is shown in Algorithm \ref{alg:gpm_lp}, where $^{\circ (p-1)}$ denotes the element-wise $(p-1)$-th power and $\circ$ denotes the element-wise product.

\begin{algorithm}  
	\small
	\caption{The GPM algorithm for $\ell_p$-based dictionary learning}  
	\label{alg:gpm_lp}  
	\begin{algorithmic}[1]
		\State Initialize $\bm{A}^{(0)*} \in \text{St}(n,m)$.
		\For{$t=0...T$}
		\State $\nabla f(\bm{A}^{(t)}) = \left( |(\bm{A}^{(t)}\bm{Y})^{\circ (p-1)}| \circ \text{sign}(\bm{A}^{(t)}\bm{Y}) \right)\bm{Y}^*$
		\State $\bm{A}^{(t+1)} = \text{Polar}(\nabla^* f(\bm{A}^{(t)}))^*$
		\EndFor
	\end{algorithmic}  
\end{algorithm}

In general, the GPM algorithm only converges to a critical point with a sub-linear rate since $\mathscr{P}$ in (\ref{eq:lp_st}) is non-convex \cite{journee2010generalized}. Nevertheless, with $\bm{X}_0$ following the Bernoulli-Gaussian model, experiments show that the GPM algorithm converges to the global maximizer in a very fast rate. We explain this phenomenon in the next subsection.

\subsection{Global convergence over the sphere constraint}
In this subsection, we investigate why and how the GPM algorithm converges to the global maximizer despite the non-convexity of problem $\mathscr{P}$ in (\ref{eq:lp_st}). Unfortunately, the global optimality analysis over the Stiefel manifold is extremely difficult. In fact, even whether a local maxima of $\|\bm{A}\|_4^4$ exists on $\bm{A}\in \mathbb{O}(n)$ is still an open problem \cite{ma2019complete}. We instead analyze a special case of $\mathscr{P}$ when $m=1$ with population GPM to see how the GPM algorithm converges to the global optimizer for a non-convex problem. The sphere constraint enables a fine characterization of the gradient dynamics, from which we can derive global convergence results.

In the rest of this subsection, we show the global convergence of the population\footnote{Due to the concentration, the empirical gradient will be very close to the population gradient when the sample size is large. Hence, we study the population GPM here to reveal the convergence behavior.} GPM algorithm over the sphere constraint. In the population version, we consider to solve the following problem with Algorithm \ref{alg:gpm_lp_s}.

\begin{equation} \label{eq:lp_s2}
\begin{aligned}
&\underset{\bm{a}}{\text{maximize}}
& & \frac{1}{r}\mathbb{E}_{\bm{Y}}\|\bm{a}^*\bm{Y}\|_p^p = \frac{1}{r}\mathbb{E}_{\bm{X}_0}\|\bm{a}^*\bm{D}_0\bm{X}_0\|_p^p\\
&\text{ subject to } 
& &\|\bm{a}\|_2=1.
\end{aligned}
\end{equation}

\begin{algorithm}  
	\small
	\caption{The population GPM algorithm over the sphere constraint}  
	\label{alg:gpm_lp_s}  
	\begin{algorithmic}[1]
		\State Initialize $\|\bm{a}^{(0)}\|_2=1$.
		\For{$t=0...T$}
		\State $\nabla f(\bm{a}^{(t)}) = \bm{Y}\left( |(\bm{a}^{(t)*}\bm{Y})^{\circ (p-1)}| \circ \text{sign}(\bm{a}^{(t)*}\bm{Y}) \right)^*$
		\State $\bm{a}^{(t+1)} = \frac{\mathbb{E}_{\bm{Y}}[\nabla f(\bm{a}^{(t)})]}{\|\mathbb{E}_{\bm{Y}}[\nabla f(\bm{a}^{(t)})]\|_2}$
		\EndFor
	\end{algorithmic}  
\end{algorithm}

We assume  $\bm{D}_0 = \bm{I}$ without loss of generality since the orthogonal transformation has no impact on the convergence \cite{sun2015complete}. Moreover, according to the statistical analysis, the global maximizer satisfies $\|\bm{a}\|_0 = 1$ \cite{spielman2012exact}. Thus, there are $2n$ ground-truth vectors in the dictionary learning problem, i.e., $\bm{a}=\pm\bm{e}_i$. We can safely assume that the desired global maximizer is $\bm{e}_n$ and $a_n \geq |a_i|,\forall i$ at initialization.

Accordingly, given the ground-truth $\bm{e}_n$, a vector $\bm{a}$ can be decomposed into two components: a parallel component $a_n$ and a perpendicular components $\bm{a}_{-n} := \bm{a}_{1:n-1}$. The parallel component is the signal component.

To study the dynamics of the population GPM algorithm, we start by considering the case where the sequences $\{\bm{a}^{(t)}\}$ are generated by the population gradient 
\begin{align*}
	\bm{a}^{(t+1)} = \frac{\nabla F(\bm{a}^{(t)})}{\|\nabla F(\bm{a}^{(t)})\|_2},
\end{align*}
where $\nabla F(\bm{a}^{(t)})$ denotes the population gradient, given by
$$\nabla F(\bm{a}^{(t)}) = \frac{1}{r}\mathbb{E}_{\bm{Y}}\nabla f(\bm{a}^{(t)}) = c_p\mathbb{E}_{\Omega}[\|\bm{a}_{\Omega}\|_2^{p-2 }\bm{a}_{\Omega}],$$ 
where $c_p$ is a constant only related to $p$, and $\Omega$ denotes the support of a random Bernoulli vector $\bm{b} \in \mathbb{R}^n$ with $b_i \sim \text{Ber}(\theta)$.

With simple calculations, the dynamics for both the signal and orthogonal components with respect to the global maximizer $\bm{e}_n$ are given by
\begin{align*}
	&\text{Signal}: &&a^{(t+1)}_n = \frac{c_p\mathbb{E}_{\Omega}[\|\bm{a}^{(t)}_{\Omega}\|_2^{p-2 }a_{n, \Omega}]}{\|\nabla F(\bm{a}^{(t)})\|_2}\\
	&\text{Orthogonal}: &&a^{(t+1)}_{i} = \frac{c_p\mathbb{E}_{\Omega}[\|\bm{a}^{(t)}_{\Omega}\|_2^{p-2 }a_{i,\Omega}]}{\|\nabla F(\bm{a}^{(t)})\|_2}, i \neq n
\end{align*}

To simplify the analysis, we define the signal-to-orthogonal-ratio (SOR) and signal-to-orthogonal at the $i$-th coordinate ratio ($\text{SOR}_i$) as 
\begin{align*}
	\text{SOR} =  \frac{a_n}{\|\bm{a}_{-n}\|_2}, \quad \text{and} \quad  \text{SOR}_i =  \frac{a_n}{a_i}.
\end{align*}

A unique advantage for studying SOR and $\text{SOR}_i$ is that they are projection-invariant, i.e., $\frac{(\mathcal{P}_{\mathbb{S}^{n-1} }\bm{q})_n}{\|(\mathcal{P}_{\mathbb{S}^{n-1} }\bm{q})_{-n}\|_2} = \frac{q_n}{\|\bm{q}_{-n}\|_2}$ and $\frac{(\mathcal{P}_{\mathbb{S}^{n-1} }\bm{q})_n}{(\mathcal{P}_{\mathbb{S}^{n-1} }\bm{q})_i} = \frac{q_n}{q_i}$. This allows us to bypass the study of projection and makes the results easy to interpret.

The SOR and $\text{SOR}_i$ can also be viewed as an error metric since 
\begin{align*}
	\|\bm{a}- \bm{e}_n\|_2^2 = 2 - 2\sqrt{\frac{\text{SOR}^2}{\text{SOR}^2+1}  },
\end{align*}
where $\left\|\bm{a} - \bm{e}_n\right\|_2^2$ is the squared $\ell_2$ error. 
\begin{figure*}[htbp]
	\centering
	\subfigure[Convergence over the sphere. The error metric is $\min_{1\leq i \leq n} \|\bm{a} - \bm{D}_0\bm{e}_i\|_2$.]{
		\begin{minipage}[t]{0.4\linewidth}
			\centering
			\includegraphics[width=1\linewidth]{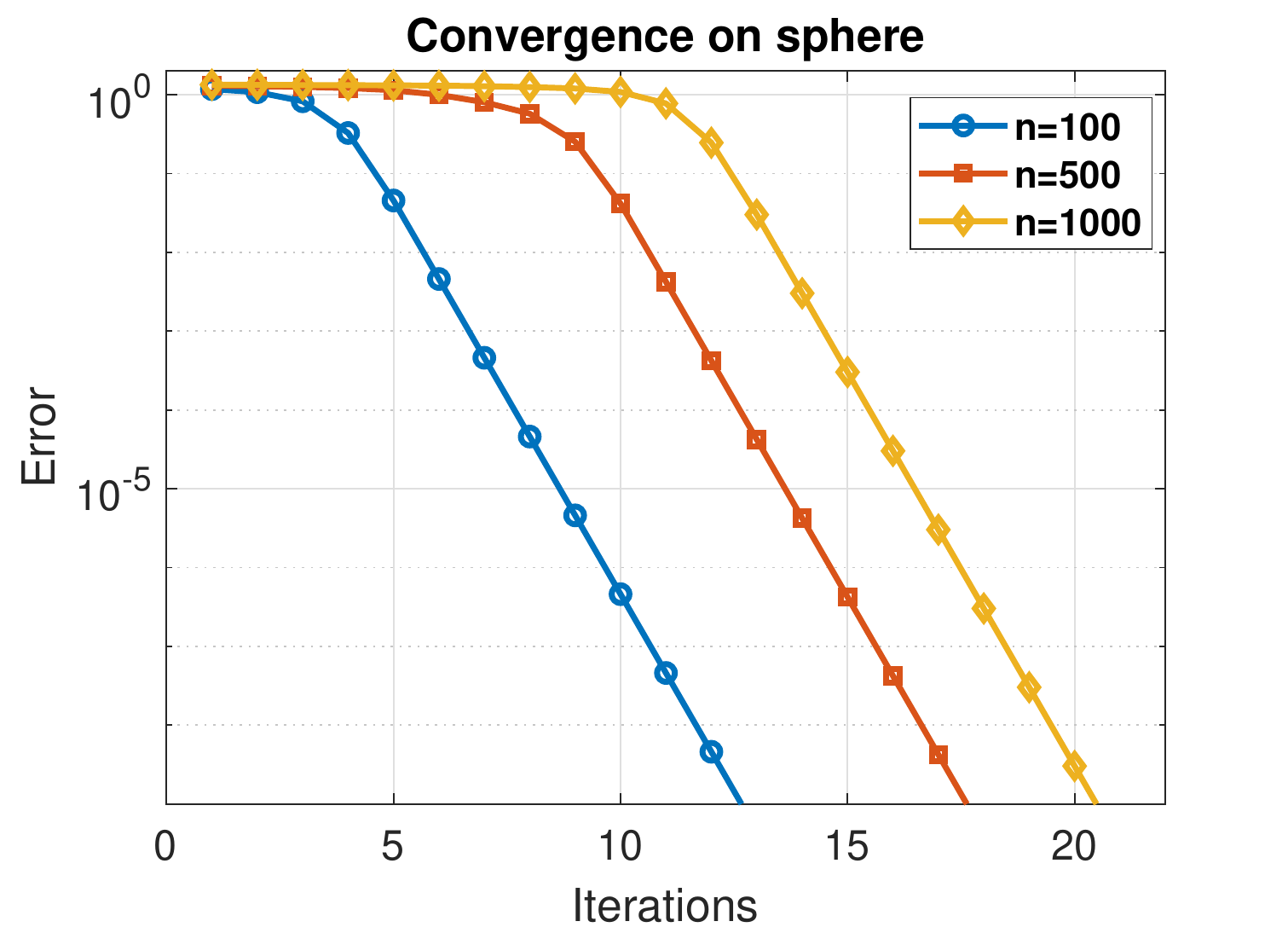}
			\label{fig:sphere} 
		\end{minipage}%
	}%
	\hspace{.1in}
	\subfigure[Convergence over the orthogonal group. The error metric is $\min_{\bm{\Pi} \in \text{SP}(n)}\frac{\|\bm{A}^*-\bm{D}_0\bm{\Pi}\|_F}{\|\bm{D}_0\|_F}$.]{
		\begin{minipage}[t]{0.4\linewidth}
			\centering
			\includegraphics[width=1\linewidth]{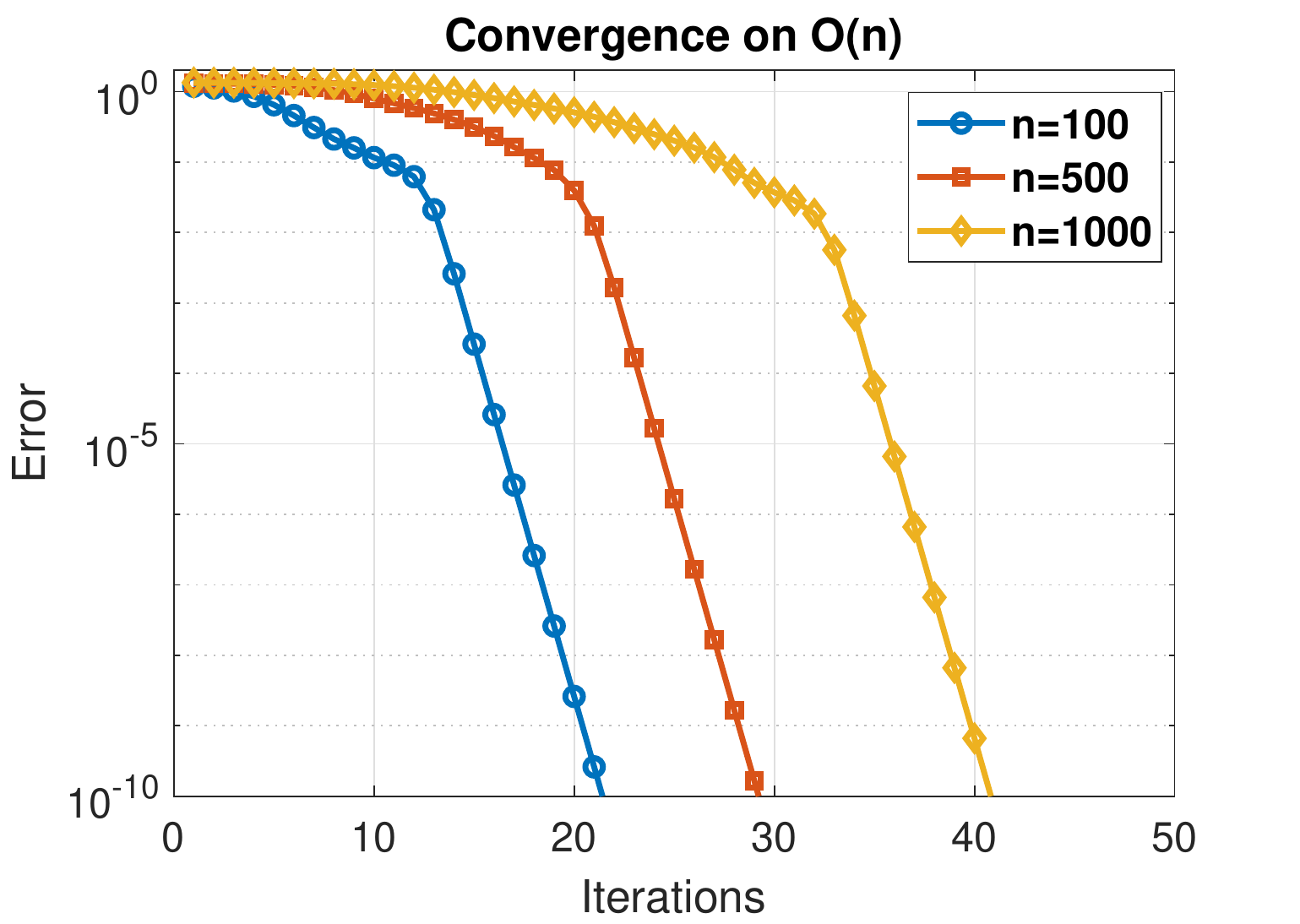}
			\label{fig:On}
		\end{minipage}%
	}%
	\centering
	\caption{The convergence of the population GPM for $\ell_4$-based formulation. The sparsity level is $\theta=0.1$. From both figures, we see the first stage is short and the second stage has a linear convergence with rate $\frac{1}{10}$.}
\end{figure*}

The following proposition shows the evolution of $\text{SOR}_i$ during the iterations, which will be used to derive the global convergence result.
\begin{prop} \label{prop:snr}
	Denote $\text{SOR}_i^{(t)}$ as the value of $\text{SOR}_i$ at the $t$-th iteration and $\bm{q} = \bm{a}^{(t)}$ as the variable at the $t$-th iteration. Then $\text{SOR}_i$ evolves as
	\begin{align*}
	\text{SOR}_i^{(t+1)} = \text{SOR}^{(t)}_i \left(1+ \tau_i(\bm{q}) \right),
	\end{align*}
	and 
	\begin{align*}
	\tau_i(\bm{q}) = \frac{\mathbb{E}_{\Omega'} \|[\bm{q}_{\Omega'},q_n]\|_2^k -  \mathbb{E}_{\Omega'} \|[\bm{q}_{\Omega'},q_i]\|_2^k }{ \frac{\theta}{1-\theta} \mathbb{E}_{\Omega'} \|[\bm{q}_{\Omega'},q_i,q_n]\|_2^k +  \mathbb{E}_{\Omega'} \|[\bm{q}_{\Omega'},q_i]\|_2^k },
	\end{align*} 
	where $\Omega' = \Omega\backslash\{n\}\backslash\{i\}$ and $k=p-2$. Two properties of $\tau_i(\bm{q})$ are listed below
	\begin{enumerate}
		\item $0 \leq \tau_i(\bm{q}) \leq \frac{1-\theta}{\theta}$ always holds and $\tau_i(\bm{q}) > 0$ if $q_n > q_i$.
		\item $\tau_i(\bm{q})$ is monotonically increasing in $q_n$ and decreasing in $q_i$. 
	\end{enumerate}
\end{prop}

\begin{rem} 
	Please refer to Section \ref{sec:sor} for the proof of this proposition.
\end{rem}

According to Proposition \ref{prop:snr}, $\text{SOR}_i$ grows at an exponential rate. Thus, Algorithm (\ref{alg:gpm_lp_s}) converges to the global maximizer of Problem (\ref{eq:lp_s2}) and the convergence is stated as below.

\begin{thm} \label{thm:conv}
	Apply Algorithm \ref{alg:gpm_lp_s} to solve problem (\ref{eq:lp_s2}) and assume $\bm{a}^{(0)}$ follows a uniform distribution over the sphere. Denote $\tau(\bm{q}) = \min_{i=1,\cdots,n-1} \tau_i(\bm{q})$, then there exists $T_{\tau} \leq  \log_{1+\tau(\bm{a}^{(0)})}\left(\sqrt{n}\right)$, $1 \leq i \leq n$, such that
	\begin{align*}
	\left\|\bm{a}^{(t)} - \bm{D}_0\bm{e}_i\right\|_2 \leq \left(1+\tau\left(\bm{a}^{(T_{\tau})}\right)\right)^{T_{\tau} - t}, \forall t \geq T_{\tau}, 
	\end{align*}
	almost surely and the convergence rate is given by $\lim_{k\rightarrow \infty} \frac{\|\bm{a}^{(k+1)} - \bm{D}_0\bm{e}_i\|_2}{\|\bm{a}^{(k)} -  \bm{D}_0\bm{e}_i\|_2} = \frac{1}{\theta}$.
\end{thm}

\begin{rem} 
	Please refer to Section \ref{sec:conv} for the proof of this theorem.
\end{rem}

From Theorem \ref{thm:conv}, we observe a two-stage convergence. Specifically, the first stage only takes $T_\tau$ iterations, which is short, and the second stage enjoys a linear convergence rate. Fig. \ref{fig:sphere} illustrates the convergence of the population GPM algorithm for the $\ell_4$-based formulation over the sphere. We clearly see a two-stage convergence in the figure: The first stage only lasts for a few iterations and the second stage has a linear convergence rate of $\frac{1}{10}$. The convergence speed is much faster than Riemannian gradient (RGD), which is the most commonly adopted method on Riemannian manifold \cite{absil2009optimization}. We provide a comparison of the GPM algorithm and RGD in Section \ref{sec:rgd}. The convergence over $\mathbb{O}(n)$ is shown in Fig. \ref{fig:On}, from which we also observe the two-stage convergence with a $\frac{1}{10}$ rate.

\section{EXPERIMENTS} \label{sec:exp}
In this section, we test the performance of the $\ell_p$-based approaches under both noiseless and noisy conditions. Specifically, we verify the advantages of the $\ell_3$-based method revealed in the theoretical analysis, especially its robustness. All the experiments are conducted with Matlab 2019a running on Intel Core i7-6700 CPU @ 3.40Ghz.

The following three benchmarks are considered:
\begin{itemize}
	\item K-SVD \cite{aharon2006k,rubinstein2008efficient}: K-SVD is a classic dictionary learning algorithm. We adopt the K-SVD toolbox \cite{rubinstein2008efficient} for an efficient implementation.
	\item Riemannian trust region (RTR) \cite{sun2015complete}: This method adopts the smoothed $\ell_1$-based formulation and Riemannian trust region algorithm \cite{absil2009optimization}, which enjoys an exact recovery guarantee in the noiseless case.
	\item $\ell_4$-based \cite{zhai2019complete}: It can be viewed as a special case of the $\ell_p$-based formulation studied in this paper when $p=4$. The MSP algorithm proposed in \cite{zhai2019complete} is adopted.
\end{itemize}

\paragraph{Scalability in the noiseless case:} We first test the scalability of different algorithms. Specifically, we compare the accuracy and running time of different methods with different dictionary sizes $n$ and sparsity levels $\theta$ in the noiseless case. For the tested problem size, RTR failed to terminate within ten hours. Hence, we do not include it as a baseline. The error and running time are taken average over 10 independent random trials. The results are shown in Table \ref{tab:lp}.

From Table \ref{tab:lp}, we see that all $\ell_p$-based methods achieve better performance and significant speedup compared to the K-SVD algorithm. As the problem size increases, the error of K-SVD increases while the error of $\ell_p$-based approaches remain stable.

We also observe that a smaller $p$ leads to a smaller error, which is consistent with the results in Theorem \ref{thm:ex_re}. Moreover, the $\ell_3$-based approach enjoys the least running time because of its lower per-iteration complexity. Thus, the $\ell_3$-based approach is the best choice in terms of both accuracy and speed for the noiseless case.

\begin{table*}[]
	\caption{The performance of different algorithms for noiseless objectives. Since the dictionary recovery is up to some signed permutations, we adopt the error metric $1-\|\bm{A}\bm{D}_0\|_4^4/n$ in \cite{zhai2019complete}, which gives $0\%$ error for a perfect recovery.}
	\centering
	\resizebox{0.9\textwidth}{!}{\begin{tabular}{|ccc|cc|cc|cc|cc|}
		\hline
		\multicolumn{3}{|c|}{Settings} & \multicolumn{2}{c|}{$\ell_3$-based} & \multicolumn{2}{c|}{$\ell_4$-based \cite{zhai2019complete}}  & \multicolumn{2}{c|}{$\ell_5$-based} & \multicolumn{2}{c|}{K-SVD\cite{rubinstein2008efficient}}  \\ \hline
	  	n &   $\theta$    &   p($\times 10^{4}$)    &    Time      &    Error     & Time    &  Error         & Time          &    Error  & Time          &    Error       \\ \hline
		100 &    0.1   &   4    &     \bf{0.8s}      &     \bf{0.056\%}      &    1.8s       &     0.21\%      &    1.7s      &    0.50\% &    61s       &   1.45\%       \\ \hline
		200 &   0.1    &   8    &      \bf{4.1s}     &      \bf{0.056\%}     &     9.3s      &     0.21\%      &     8.0s      &    0.51\% &   131s        &   3.03\%        \\ \hline
		400 &   0.1    &   16    &     \bf{35s}      &      \bf{0.056\%}     &     50s      &      0.21\%     &      41s     &    0.50\% &   315s        &   6.45\%       \\ \hline
		100 &   0.3    &   4    &     \bf{1.2s}      &      \bf{0.094\%}     &     3.4s      &     0.34\%      &    3.1s       &    0.84\%  &   98s        &    2.60\%      \\ \hline
		200 &   0.3    &    8   &     \bf{10s}      &      \bf{0.094\%}     &     18s      &      0.35\%     &      15s     &     0.85\% &    215s       &    6.41\%      \\ \hline
		400 &   0.3    &    16   &     \bf{91s}      &      \bf{0.096\%}     &      122s     &      0.35\%     &     146s      &    1.00\%   &   589s        &    8.25\%      \\ \hline
	\end{tabular}}
	\label{tab:lp}
\end{table*}

\begin{table*}[]
	\caption{The performance of different algorithms under Gaussian noise. We set sparsity level $\theta=0.3$.}
	\centering
	\resizebox{0.9\textwidth}{!}{\begin{tabular}{|ccc|cc|cc|cc|cc|}
			\hline
			\multicolumn{3}{|c|}{Settings} & \multicolumn{2}{c|}{$\ell_3$-based} & \multicolumn{2}{c|}{$\ell_4$-based \cite{zhai2019complete}}  & \multicolumn{2}{c|}{RTR\cite{sun2015complete}} & \multicolumn{2}{c|}{K-SVD\cite{rubinstein2008efficient}}  \\ \hline
			n &   p($\times 10^{4}$)    &  $\sigma$     &    Time      &    Error     & Time    &  Error         & Time          &    Error  & Time          &    Error       \\ \hline
			32 &    1   &   0    &     \bf{0.05s}      &     0.10\%      &    0.24s       &     0.4\%      &    100s      &    \bf{0.05\%} &    25s       &   0.2\%       \\ \hline
			32 &   1   &   0.2    &      \bf{0.05s}     &      \bf{0.27\%}     &     0.24s      &     0.6\%      &     250s      &    0.5\% &   25s        &   0.37\%        \\ \hline
			32 &   1    &   0.4    &     \bf{0.1s}      &      \bf{0.79\%}     &     0.36s      &      1.2\%     &      577s     &  4.27\%   &   25s        &   2.0\%       \\ \hline
			32 &   1   &    0.6   &     \bf{0.2s}      &      \bf{2.3\%}     &     0.7s      &     3.4\%     &      823s     &   57.4\%   &    25s       &    57.4\%      \\ \hline
			100 &   4    &    0   &     \bf{1.2s}      &      0.1\%     &      3.4s     &      0.35\%     &     863s      &    \bf{0.05\%}   &   98s        &    2.60\%      \\ \hline
			100 &   4    &    0.2   &     \bf{2.2s}      &      \bf{0.2\%}    &      4.2s     &      0.5\%     &     1643s      &  0.3\%     &   104s        &    3.46\%      \\ \hline
			100 &   4    &    0.4   &     \bf{3.5s}      &      \bf{0.6\%}     &      6.1s     &      1.1\%     &     3796s      &   5.26\%   &   105s        &    3.56\%      \\ \hline
			100 &   4    &    0.6   &     \bf{8.4s}      &      \bf{1.95\%}     &      13.5s     &      2.63\%     &     5412s      &  50.5\%    &   104s        &    51.26\%      \\ \hline
	\end{tabular}}
	\label{tab:lp_gn}
\end{table*}

\begin{table*}[]
	\caption{The performance of different algorithms under sparse noise. We set sparsity level $\theta=0.3$. The sparsity of noise is set to $\vartheta = 0.1$.}
	\centering
	\resizebox{0.9\textwidth}{!}{\begin{tabular}{|ccc|cc|cc|cc|cc|}
			\hline
			\multicolumn{3}{|c|}{Settings} & \multicolumn{2}{c|}{$\ell_3$-based} & \multicolumn{2}{c|}{$\ell_4$-based \cite{zhai2019complete}}  & \multicolumn{2}{c|}{RTR\cite{sun2015complete}} & \multicolumn{2}{c|}{K-SVD\cite{rubinstein2008efficient}}  \\ \hline
			n &   p($\times 10^{4}$)    &  $\sigma$     &    Time      &    Error     & Time    &  Error         & Time          &    Error  & Time          &    Error       \\ \hline
			32 &    1   &   0.5    &     \bf{0.06s}      &     0.20\%      &    0.25s       &     0.57\%      &    362s      &    \bf{0.10\%} &    25s       &   0.37\%       \\ \hline
			32 &   1   &   1    &      \bf{0.09s}     &      \bf{0.50}\%     &     0.35s      &     0.93\%      &     421s      &    1.4\% &   25s        &   2.0\%        \\ \hline
			32 &   1    &   1.5    &     \bf{0.14s}      &      \bf{1.65\%}     &     0.47s      &      2.26\%     &      420s     &    13.4\% &   25s        &   57.4\%       \\ \hline
			100 &   1   &    0.5   &     \bf{1.9s}      &      \bf{0.20\%}     &     4.0s      &     0.40\%     &      1649s     &     \bf{0.2\%} &    104s       &    3.04\%      \\ \hline
			100 &   4    &    1   &     \bf{2.6s}      &      \bf{0.40\%}     &      5.5s     &      0.80\%     &     2737s      &    1.3\%   &   105s        &    3.55\%      \\ \hline
			100 &   4    &    1.5   &     \bf{4.6s}      &      \bf{1.02\%}    &      7.7s     &      1.49\%     &     5395s      &    36.8\%   &   104s        &    5.83\%      \\ \hline
	\end{tabular}}
	\label{tab:lp_sn}
\end{table*}

\paragraph{Gaussian noise:} A small Gaussian noise usually appears in dictionary learning applications.  We consider the noisy observation $\bm{Y}_N = \bm{Y} + \sigma \bm{G}$, where $G_{i,j} \sim \mathcal{N}(0,1)$. The results of different algorithms are shown in Table \ref{tab:lp_gn}. While RTR achieves the smallest error on a clean objective (i.e., the cases with $\sigma=0$), it is very time-consuming. When the noise is large, both RTR and K-SVD fail to recover the dictionary while the $\ell_3$ and $\ell_4$-based methods still have a relatively low error. This is because for the $\ell_3$ and $\ell_4$-norm, the regions around the global maximizers are very flat, leading to a higher tolerance to noise. Table \ref{tab:lp_gn} also shows that the $\ell_3$-based method is the most time-efficient in all cases and achieves the best performance when the noise is present.

\paragraph{Sparse corruptions:}
Sparse corruptions are another kind of noise usually present in images \cite{candes2011robust}. We consider the noisy observation $\bm{Y}_S = \bm{Y} + \sigma \bm{B} \circ \bm{R}$, where $B_{i,j} \sim \text{Ber}(\vartheta)$ and $R_{i,j}$ has equal probability to be $-1$ and $1$. The results are shown in Table \ref{tab:lp_sn}. The performance of different methods are similar to the Gaussian noise case. RTR performs the best when the noise is relatively small but very unstable as the noise increases. The $\ell_3$ and $\ell_4$-based methods are more stable as the noise increases and the $\ell_3$-based method is still the most time-efficient and most robust one.

All the above experiments demonstrate that the $\ell_3$-based approach is more time efficient and robust than existing methods, and thus it is a preferred method to use in practice.

\section{RELATED WORKS}
In this section, we discuss some existing literature related to our work.

\paragraph{Applications of $\ell_p$-norm:} The $\ell_p$-norm plays an important role in machine learning. $\ell_1$-norm can induce sparsity and has been widely applied in compressive sensing \cite{foucart2017mathematical,candes2011robust}. $\ell_2$-norm has strong geometric connections to eigenvectors as well as variance. Thus, it is often used in different kinds of principle component pursuits \cite{vidal2015generalized}. Maximizing $\ell_{2k+2}$-norm can thus promote the spikiness, which has found applications in independent component analysis \cite{hyvarinen1997fast}, sparse blind deconvolution \cite{li2018global,zhang2019structured}, and dictionary learning \cite{zhai2019complete,qu2019geometric}. General $\ell_p$-norms have also been applied for deconvolution applications in geophysics \cite{debeye1990lp,nose2014On}, but with little theory. Our study showed the importance of investigating more general $\ell_p$-norm based methods, and developed general frameworks for effective theoretical analysis and algorithm design. In particular, $\ell_3$-norm maximization stands out as a promising method for solving dictionary learning problems.

\paragraph{Independent Component Analysis:} ICA factors a data matrix $\bm{Y}$ as $\bm{Y} = \bm{A}\bm{X}$ such that $\bm{A}$ is square with $\bm{X}$ as independent as possible. One popular contrast function for ICA is the kurtosis \cite{hyvarinen1997fast}, which has the same formulation with (\ref{eq:lp_s2}) when $p=4$. It suggests that our general $\ell_p$-based formulation and analysis might be able to extend to ICA. The major difference is that ICA is to find statistically independent components while dictionary learning is to find the sparsest representation of the given data. As is revealed in Section 1.5 of \cite{sun2015complete}, it is the sparsity rather than independence shapes the benign landscape of the orthogonal dictionary learning problem.


\section{CONCLUSIONS}

In this paper, we considered the $\ell_p$-based ($p>2,p\in \mathbb{N}$) formulations for the orthogonal dictionary learning problem. We showed that the global maximizers of these formulations are very close to the true dictionary in both noiseless and noisy cases. In addition, we developed an efficient algorithm based on the generalized power method, and demonstrated its fast global convergence. We further conducted various experiments to show the benefits of adopting the $\ell_3$-based approach.

\section*{Acknowledgement}

We would like to thank professor Yi Ma of Berkeley EECS Department for his lectures and talks at Tsinghua-Berkeley Shenzhen Institute and Yuexiang Zhai of Berkeley for stimulating discussions during preparation of this manuscript.

\newpage

\bibliographystyle{plain}
\bibliography{Reference}

\begin{thebibliography}{10}

\bibitem{absil2009optimization}
P-A Absil, Robert Mahony, and Rodolphe Sepulchre.
\newblock {\em Optimization algorithms on matrix manifolds}.
\newblock Princeton University Press, 2009.

\bibitem{aharon2006k}
Michal Aharon, Michael Elad, and Alfred Bruckstein.
\newblock {K-SVD}: An algorithm for designing overcomplete dictionaries for
  sparse representation.
\newblock {\em IEEE Transactions on Signal Processing}, 54(11):4311--4322,
  2006.

\bibitem{bai2018subgradient}
Yu~Bai, Qijia Jiang, and Ju~Sun.
\newblock Subgradient descent learns orthogonal dictionaries.
\newblock In {\em International Conference on Learning Representations}, 2019.

\bibitem{bao2013fast}
Chenglong Bao, Jian-Feng Cai, and Hui Ji.
\newblock Fast sparsity-based orthogonal dictionary learning for image
  restoration.
\newblock In {\em IEEE International Conference on Computer Vision}, pages
  3384--3391, 2013.

\bibitem{bathe2013subspace}
Klaus-J{\"u}rgen Bathe.
\newblock The subspace iteration method--revisited.
\newblock {\em Computers \& Structures}, 126:177--183, 2013.

\bibitem{candes2011robust}
Emmanuel~J Cand{\`e}s, Xiaodong Li, Yi~Ma, and John Wright.
\newblock Robust principal component analysis?
\newblock {\em Journal of the ACM}, 58(3):11, 2011.

\bibitem{debeye1990lp}
HWJ Debeye and P~Van~Riel.
\newblock $\ell_p$-norm deconvolution.
\newblock {\em Geophysical Prospecting}, 38(4):381--403, 1990.

\bibitem{elad2010high}
M.~Elad.
\newblock {\em Sparse and Redundant Representations: From Theory to
  Applications in Signal and Image Processing}.
\newblock New York, NY, USA: Springer-Verlag, 2010.

\bibitem{foucart2017mathematical}
Simon Foucart and Holger Rauhut.
\newblock A mathematical introduction to compressive sensing.
\newblock {\em Bulletin of the American Mathematical Society}, 54:151--165,
  2017.

\bibitem{gilboa2018efficient}
Dar Gilboa, Sam Buchanan, and John Wright.
\newblock Efficient dictionary learning with gradient descent.
\newblock In {\em International Conference on Machine Learning}, pages
  2252--2259, 2019.

\bibitem{hyvarinen1997fast}
Aapo Hyv{\"a}rinen and Erkki Oja.
\newblock A fast fixed-point algorithm for independent component analysis.
\newblock {\em Neural computation}, 9(7):1483--1492, 1997.

\bibitem{journee2010generalized}
Michel Journ{\'e}e, Yurii Nesterov, Peter Richt{\'a}rik, and Rodolphe
  Sepulchre.
\newblock Generalized power method for sparse principal component analysis.
\newblock {\em Journal of Machine Learning Research}, 11:517--553, 2010.

\bibitem{li2018global}
Yanjun Li and Yoram Bresler.
\newblock Global geometry of multichannel sparse blind deconvolution on the
  sphere.
\newblock In {\em Advances in Neural Information Processing Systems}, pages
  1132--1143, 2018.

\bibitem{ma2019complete}
Yi~Ma.
\newblock Complete dictionary learning via $\ell_4$-norm maximization over the
  orthogonal group.
\newblock In {\em International Conference on Computer Vision Workshops}, 2019.

\bibitem{nose2014On}
Kenji Nose-Filho and Jo{\~a}o~MT Romano.
\newblock On $\ell_p$-norm sparse blind deconvolution.
\newblock In {\em IEEE International Workshop on Machine Learning for Signal
  Processing}, pages 1--6. IEEE, 2014.

\bibitem{olshausen1996emergence}
Bruno~A Olshausen and David~J Field.
\newblock Emergence of simple-cell receptive field properties by learning a
  sparse code for natural images.
\newblock {\em Nature}, 381(6583):607, 1996.

\bibitem{qu2019geometric}
Qing Qu, Yuexiang Zhai, Xiao Li, Yuqian Zhang, and Zhihui Zhu.
\newblock Geometric analysis of nonconvex optimization landscapes for
  overcomplete learning.
\newblock In {\em International Conference on Learning Representations}, 2020.

\bibitem{rubinstein2008efficient}
Ron Rubinstein, Michael Zibulevsky, and Michael Elad.
\newblock Efficient implementation of the {K-SVD} algorithm using batch
  orthogonal matching pursuit.
\newblock Technical report, Computer Science Department, Technion, 2008.

\bibitem{spielman2012exact}
Daniel~A Spielman, Huan Wang, and John Wright.
\newblock Exact recovery of sparsely-used dictionaries.
\newblock In {\em Conference on Learning Theory}, pages 1--37, 2012.

\bibitem{sun2015nonconvex}
Ju~Sun, Qing Qu, and John Wright.
\newblock When are nonconvex problems not scary?
\newblock {\em arXiv preprint arXiv:1510.06096}, 2015.

\bibitem{sun2015complete}
Ju~Sun, Qing Qu, and John Wright.
\newblock Complete dictionary recovery over the sphere {I}: Overview and the
  geometric picture.
\newblock {\em IEEE Transactions on Information Theory}, 63(2):853--884, 2017.

\bibitem{vershynin2018high}
Roman Vershynin.
\newblock {\em High-dimensional probability: An introduction with applications
  in data science}.
\newblock Cambridge University Press, 2018.

\bibitem{vidal2015generalized}
Ren{\'e} Vidal, Yi~Ma, and S~Shankar Sastry.
\newblock {\em Generalized principal component analysis}.
\newblock Springer, 2016.

\bibitem{wang2019unique}
Yu~Wang, Siqi Wu, and Bin Yu.
\newblock Unique sharp local minimum in $\ell_1$-minimization complete
  dictionary learning.
\newblock {\em Journal of Machine Learning Research}, 21:1--52, 2020.

\bibitem{zhai2019understanding}
Yuexiang Zhai, Hermish Mehta, Zhengyuan Zhou, and Ma~Yi.
\newblock Understanding $\ell^4$-based dictionary learning: Interpretation,
  stability, and robustness.
\newblock In {\em International Conference on Learning Representations}, 2020.

\bibitem{zhai2019complete}
Yuexiang Zhai, Zitong Yang, Zhenyu Liao, John Wright, and Yi~Ma.
\newblock Complete dictionary learning via $\ell_4$-norm maximization over the
  orthogonal group.
\newblock {\em arXiv preprint arXiv:1906.02435}, 2019.

\bibitem{zhang2019structured}
Yuqian Zhang, Han-Wen Kuo, and John Wright.
\newblock Structured local optima in sparse blind deconvolution.
\newblock {\em IEEE Transactions on Information Theory}, 66(1):419--452, 2020.

\end{thebibliography}

\newpage
\onecolumn
\appendix
\noindent {\Large \textbf{Supplementary Materials}}
\section{Additional Experiments}
\subsection{Comparison between GPM and RGD} \label{sec:rgd}
In this subsection, we present the comparison between the population generalized power method (GPM) and population Riemannian gradient  (RGD) for the $\ell_p$-based objective. RGD \cite{absil2009optimization} is a popular method for optimization over manifold, which is known for its low computational cost \cite{bai2018subgradient}. In the experiment, we set $\theta=0.3$ and $p=4$ to compare the convergence speed of the two methods in different scales. For RGD, we fix the step size as $\frac{1}{4}$. The results on the sphere and orthogonal group are shown in Fig. \ref{fig:sphere_RGD} and Fig. \ref{fig:On_RGD} respectively. From both figures, we see that the convergence speed of GPM is much faster than RGD, especially when the problem size is large. In addition, we observe both GPM and RGD have a two stage convergence.

\begin{figure*}[htbp]
	\centering
	\subfigure[Convergence over the sphere. The error metric is $\min_{1\leq i \leq n} \|\bm{a} - \bm{D}_0\bm{e}_i\|_2$.]{
		\begin{minipage}[t]{0.47\linewidth}
			\centering
			\includegraphics[width=1\linewidth]{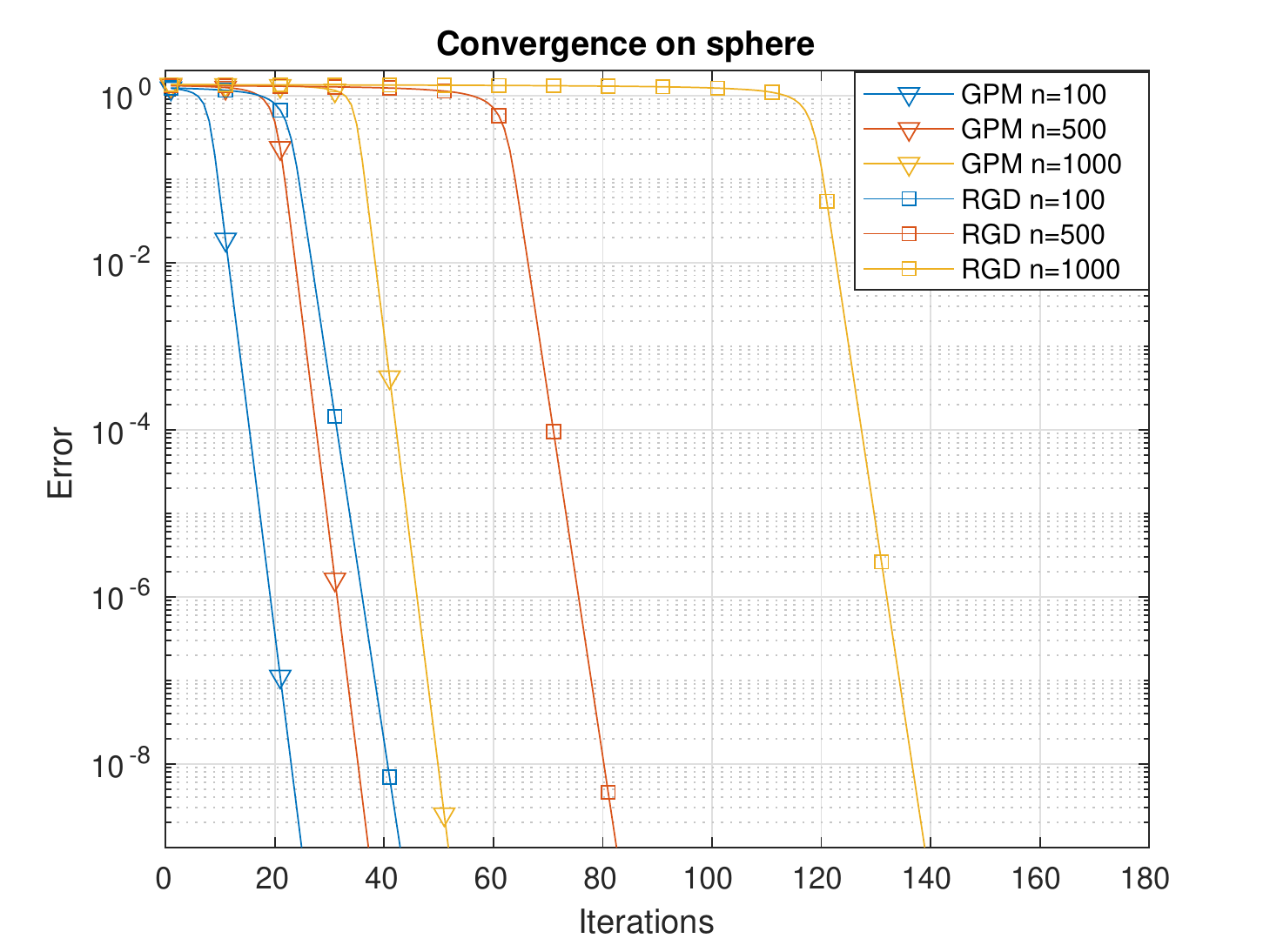}
			\label{fig:sphere_RGD} 
		\end{minipage}%
	}%
	\hspace{.1in}
	\subfigure[Convergence over the orthogonal group. The error metric is $\min_{\bm{\Pi} \in \text{SP}(n)}\frac{\|\bm{A}^*-\bm{D}_0\bm{\Pi}\|_F}{\|\bm{D}_0\|_F}$.]{
		\begin{minipage}[t]{0.47\linewidth}
			\centering
			\includegraphics[width=1\linewidth]{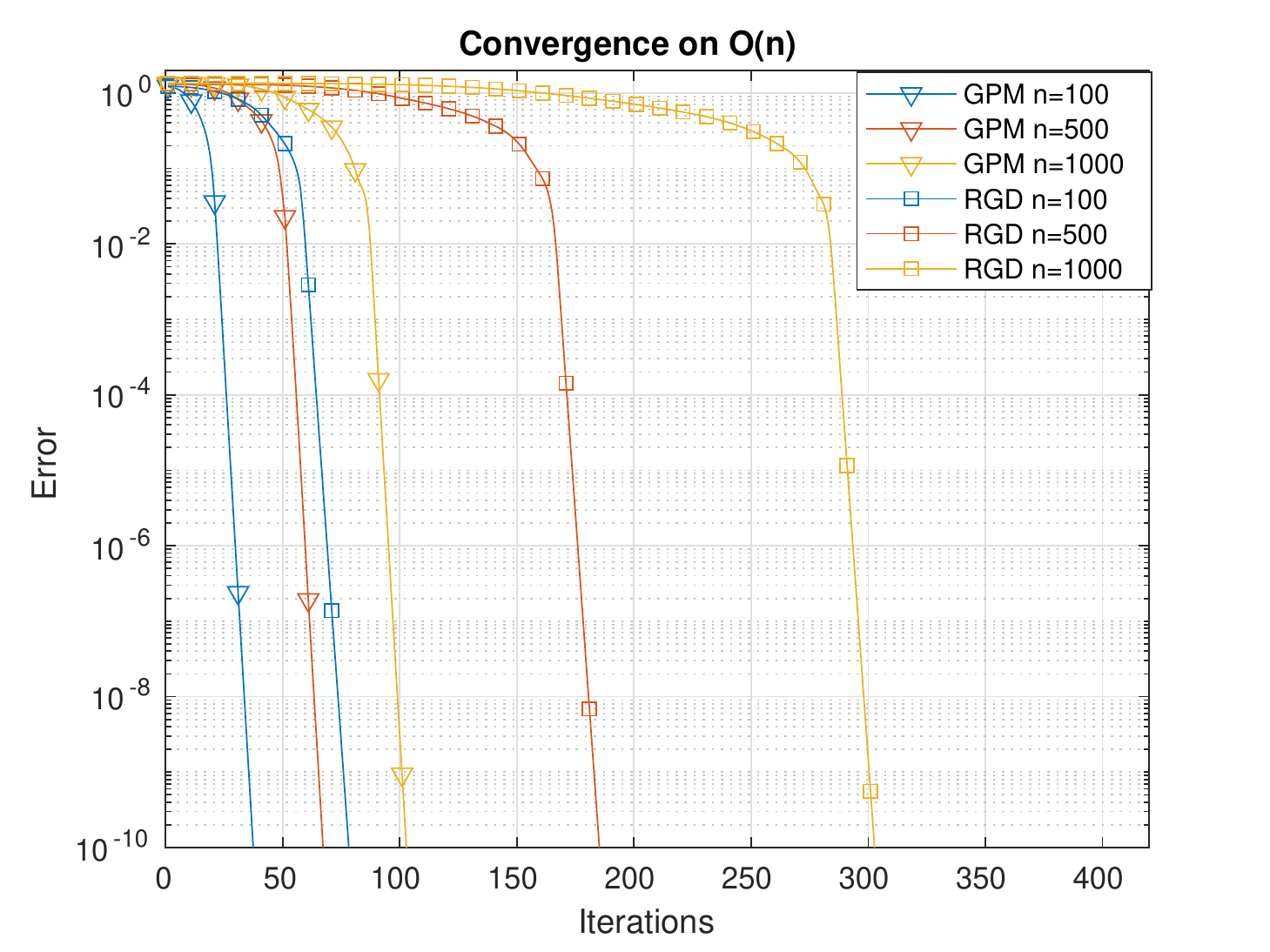}
			\label{fig:On_RGD}
		\end{minipage}%
	}%
	\centering
	\caption{The convergence of the population GPM for $\ell_4$-based formulation.}
\end{figure*}

\subsection{Phase Transition Heatmaps}
\begin{figure*}[htbp]
	\centering
	\subfigure[Fix $p=3$ and $\sigma=0.2$, varying $r$ and $\theta$.]{
		\begin{minipage}[t]{0.3\linewidth}
			\centering
			\includegraphics[width=1\linewidth]{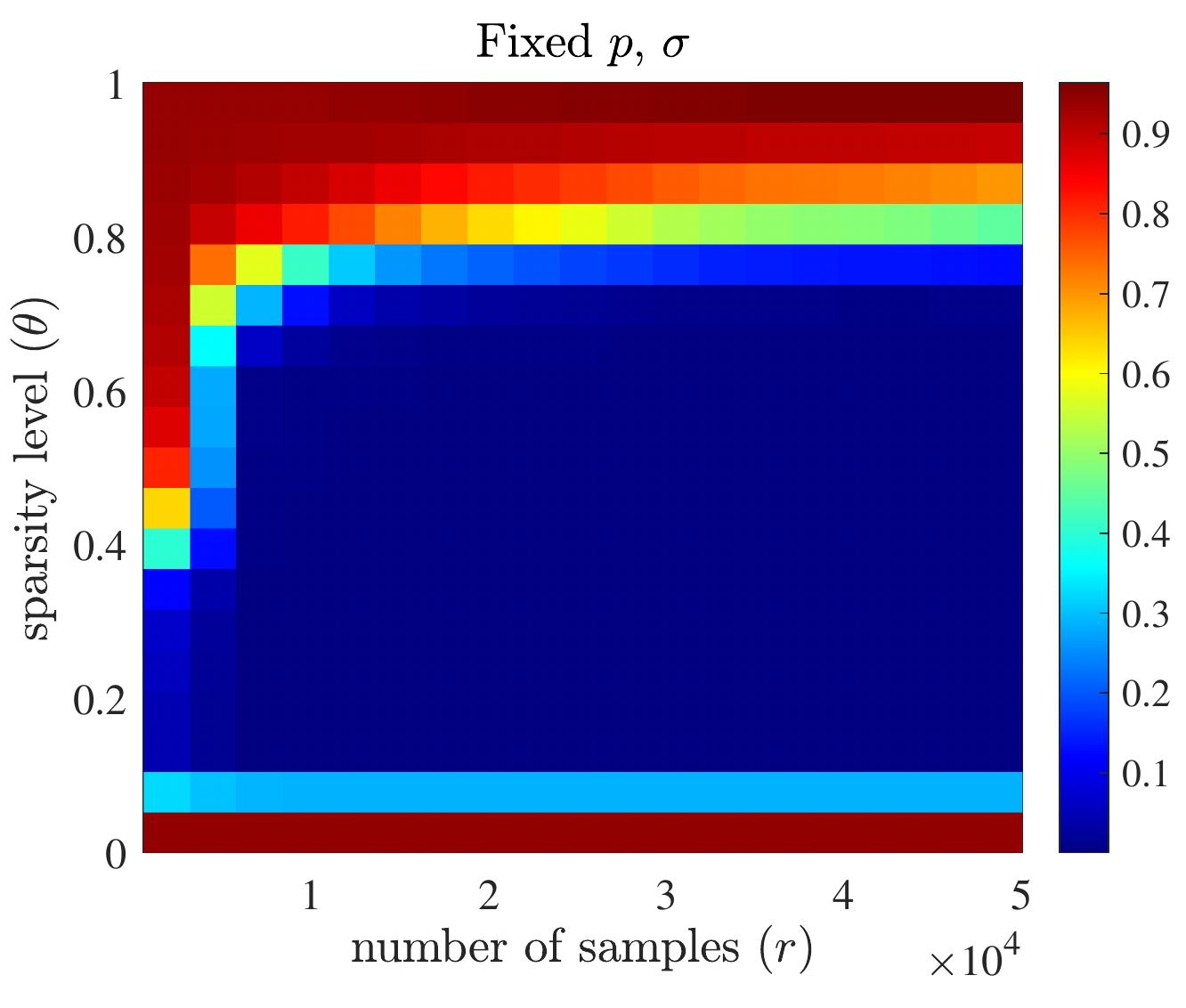}
		\end{minipage}%
	}%
	\subfigure[Fix $p=3$ and $\theta=0.3$, varying $r$ and $\sigma$.]{
		\begin{minipage}[t]{0.3\linewidth}
			\centering
			\includegraphics[width=1\linewidth]{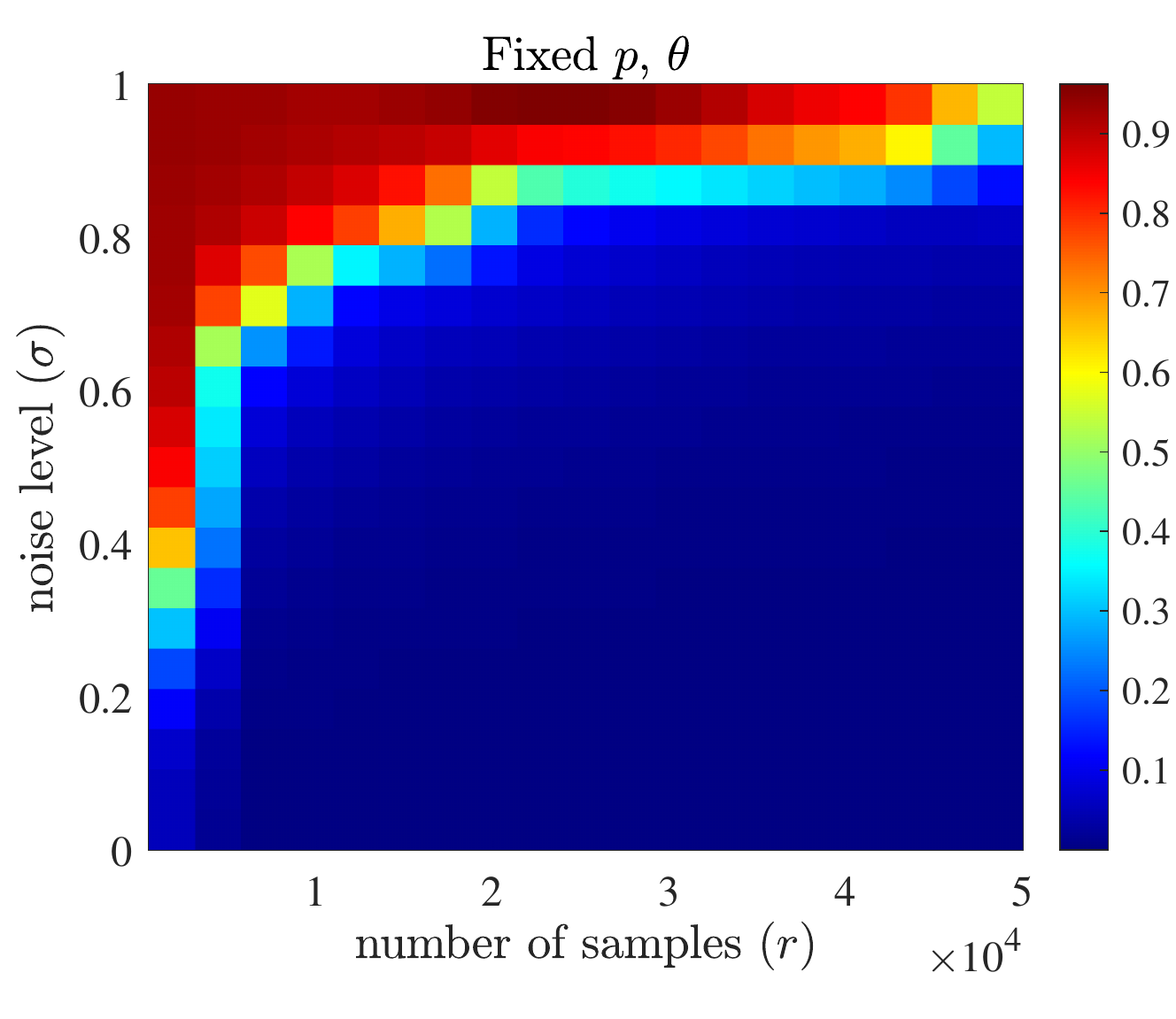}
		\end{minipage}%
	}%
	\subfigure[Fix $r=10000$ and $\sigma=0.2$, varying $p$ and $\theta$.]{
		\begin{minipage}[t]{0.3\linewidth}
			\centering
			\includegraphics[width=1\linewidth]{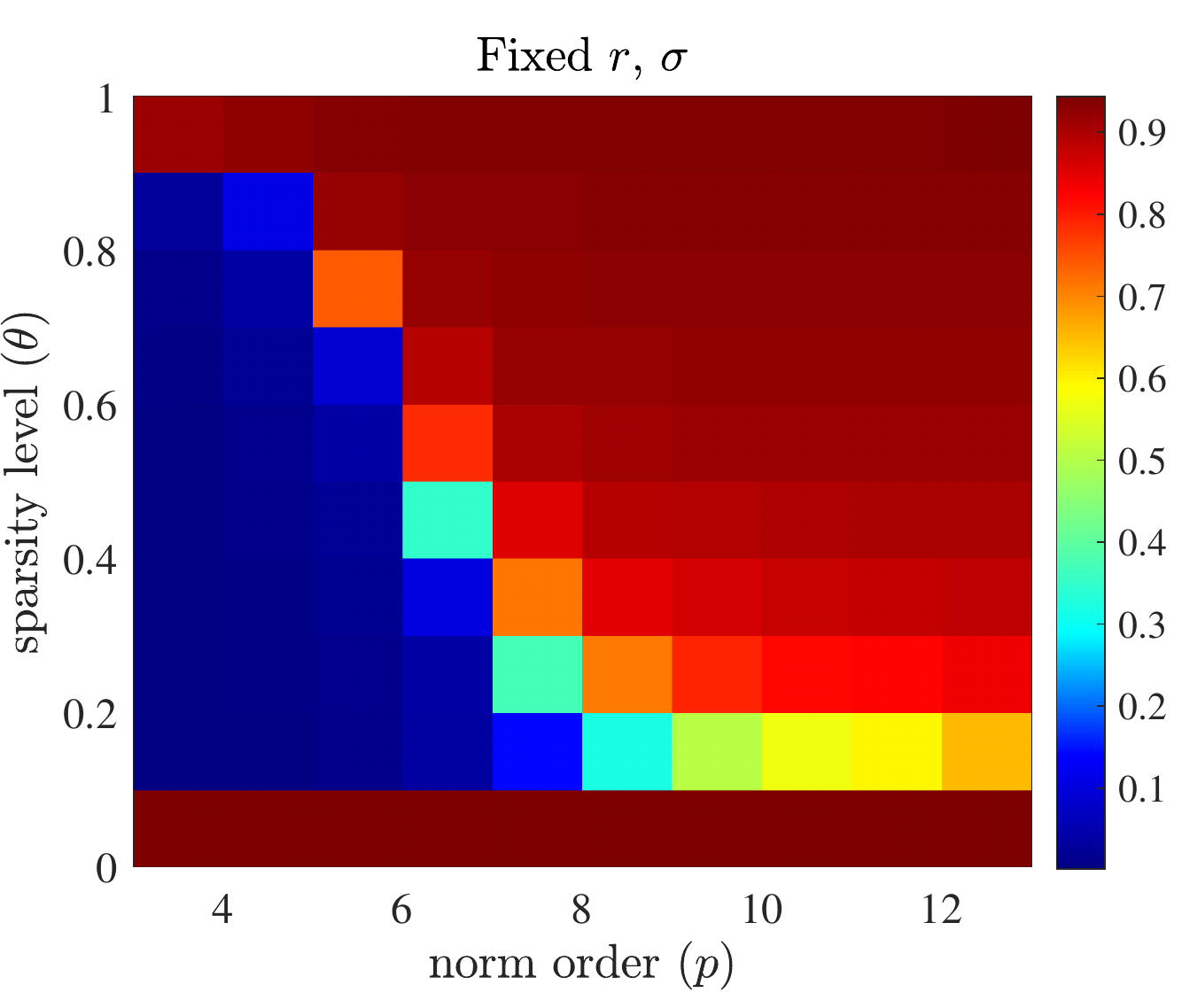}
		\end{minipage}%
	}%

	\subfigure[Fix $r=10000$ and $\theta=0.3$, varying $p$ and $\sigma$.]{
		\begin{minipage}[t]{0.3\linewidth}
			\centering
			\includegraphics[width=1\linewidth]{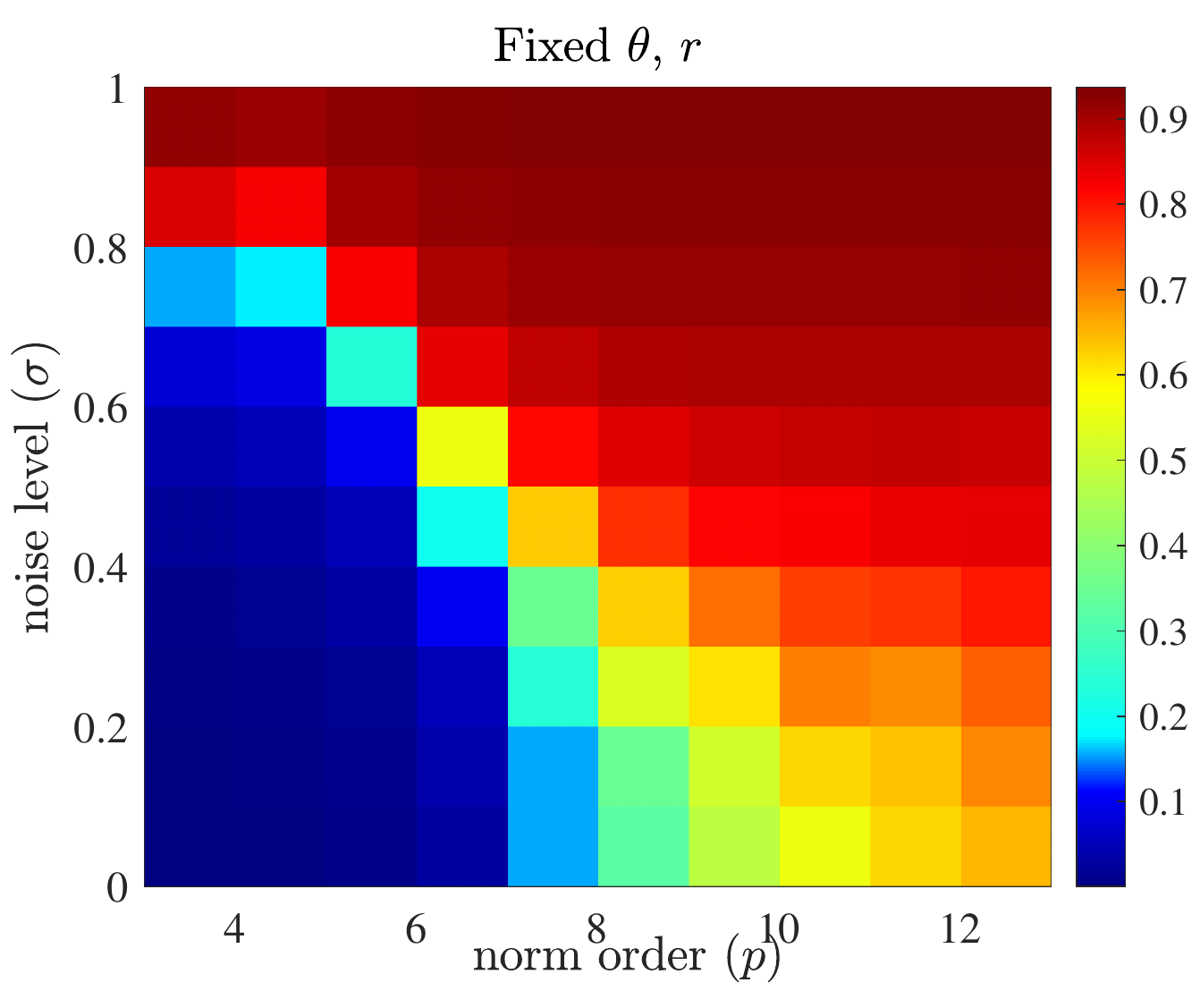}
			\label{fig:fixpsigma}
		\end{minipage}%
	}%
	\hspace{.1in}
	\subfigure[Fix $\sigma=0.2$ and $\theta=0.3$, varying $p$ and $r$.]{
		\begin{minipage}[t]{0.3\linewidth}
			\centering
			\includegraphics[width=1\linewidth]{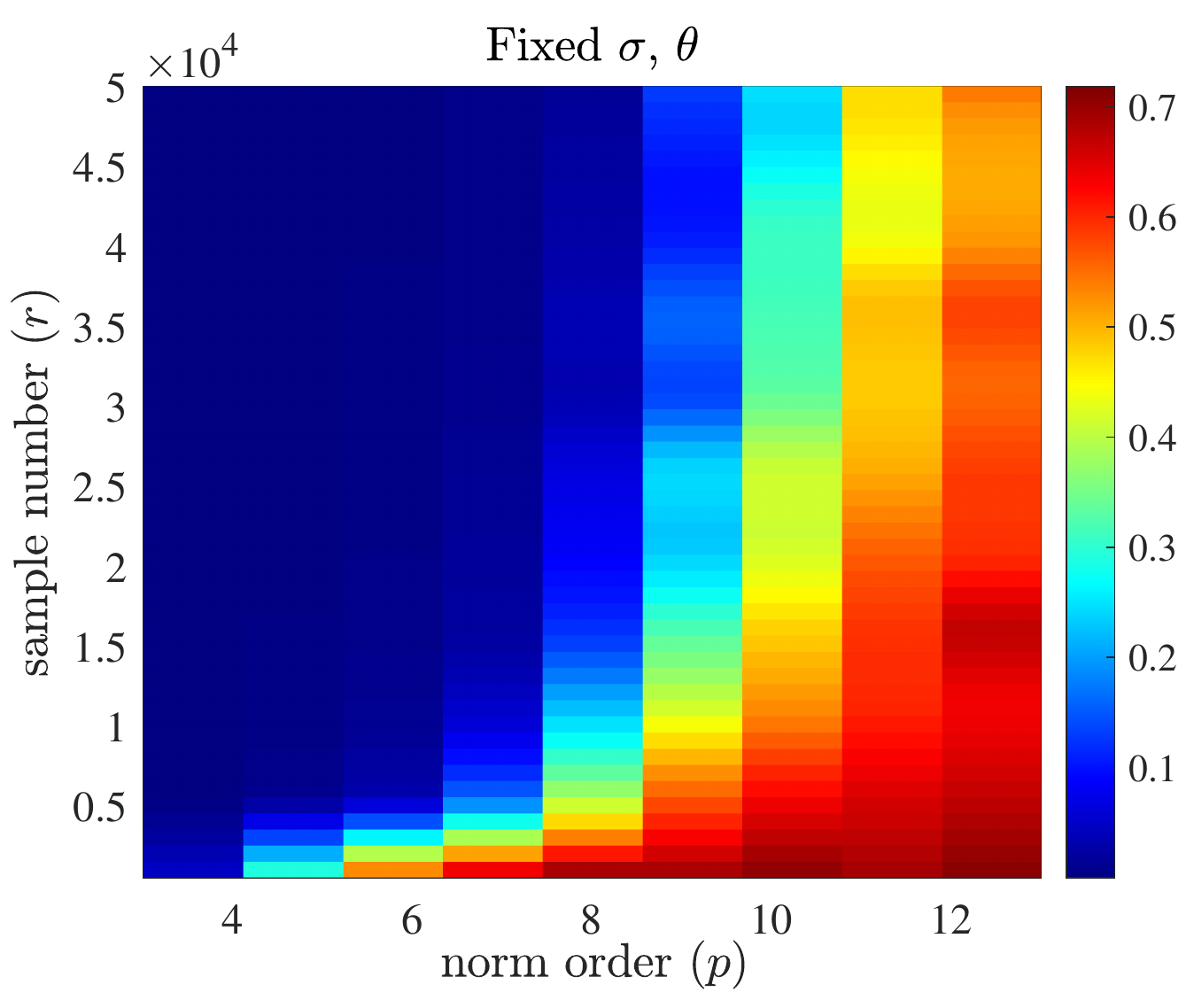}
		\end{minipage}%
	}%
	\hspace{.1in}
	\subfigure[Fix $p=3$ and $r=10000$, varying $p$ and $r$.]{
		\begin{minipage}[t]{0.3\linewidth}
			\centering
			\includegraphics[width=1\linewidth]{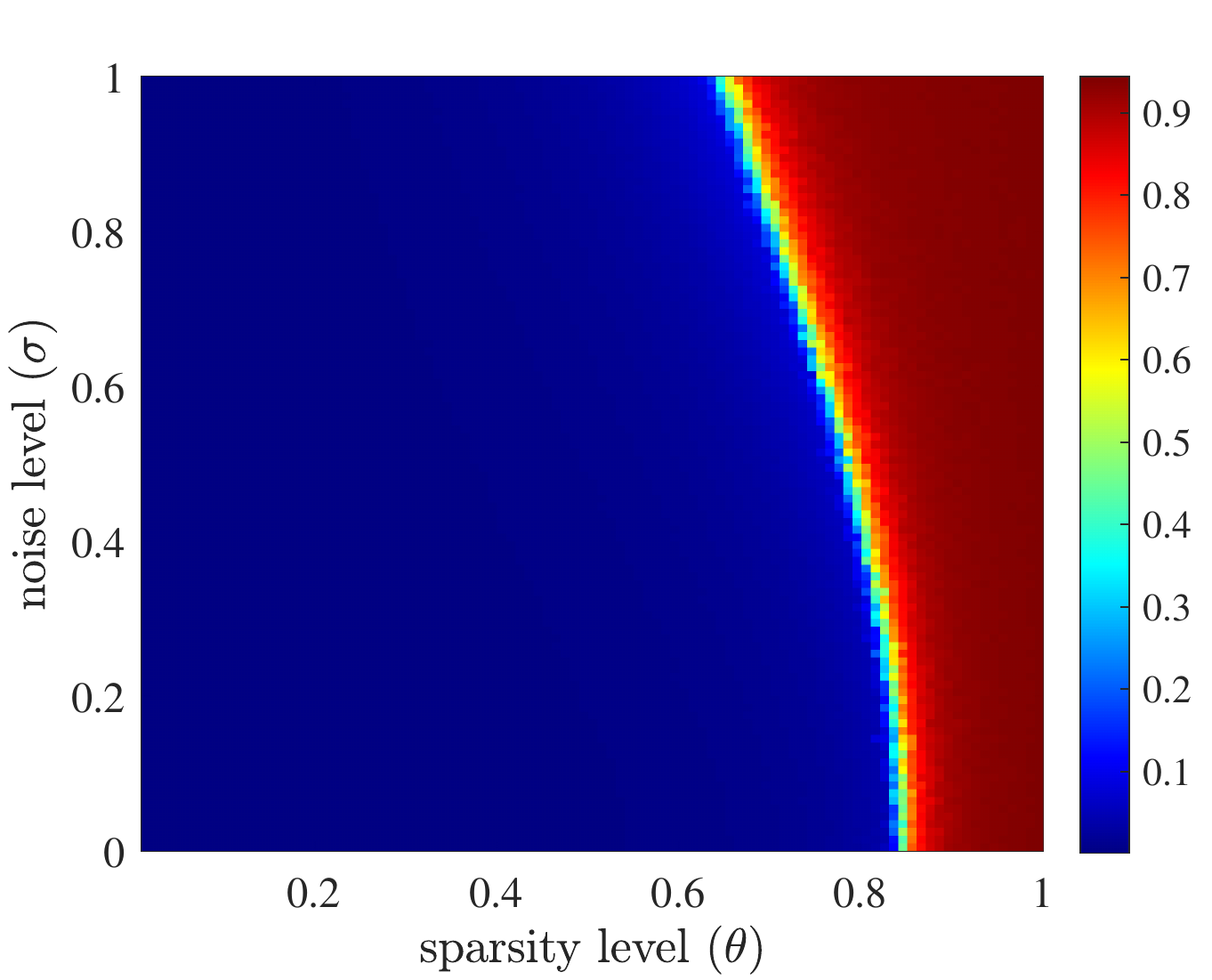}
		\end{minipage}%
	}%
	
	\centering
	\caption{Phase transitions when $n=50$.}
	\label{fig:pt_hm}
\end{figure*}

There are complicated interactions between the number of samples, $r$, the sparsity level, $\theta$, the order of the norm, $p$, and the variance, $\sigma$ of the additive Gaussian noise. For a more comprehensive illustration, we plot the heatmap by fixing two variables among $r$, $\theta$, $p$ and $\sigma$, and varying the other two to show these interactions. The results are shown in Fig. \ref{fig:pt_hm}, where the color reflects the error. 

\subsection{Experiments on Real Data}
In this subsection, we test the performance on the MNIST dataset to check whether $\ell_p$-based methods can learn a sparse representation on real images. We include principal component pursuit (PCA) as a benchmark. For each method, we take the top $5$ bases to recover the original images. The results are shown in Fig. \ref{fig:mnist}. With top $5$ bases, the recovered images by PCA can hardly be recognized while we can recognize the images recovered by $\ell_3$ and $\ell_4$-based method. This suggests that a sparse representation is learned. In addition, the recovered images via $\ell_3$-based method is a bit more identifiable than the images via $\ell_4$-based method from the figure.

\begin{figure*}[htbp]
	\centering
	\subfigure[Original images.]{
		\begin{minipage}[t]{0.5\linewidth}
			\centering
			\includegraphics[width=0.8\linewidth]{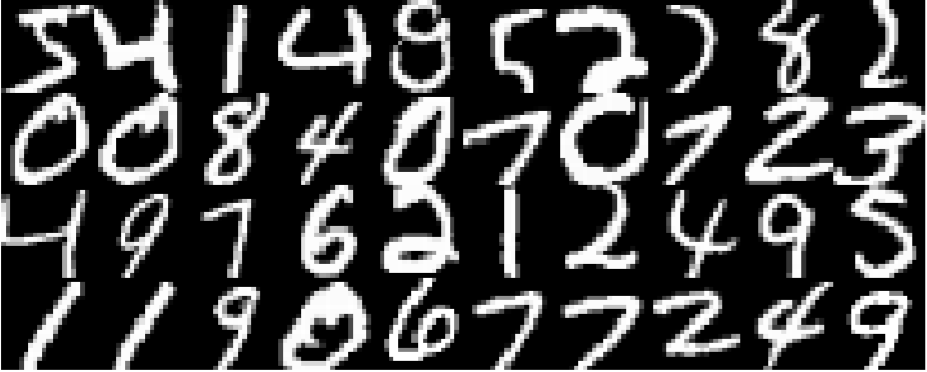}
		\end{minipage}%
	}%
	\subfigure[Recovered images via $\ell_3$-based method with top $5$ bases.]{
		\begin{minipage}[t]{0.5\linewidth}
			\centering
			\includegraphics[width=0.8\linewidth]{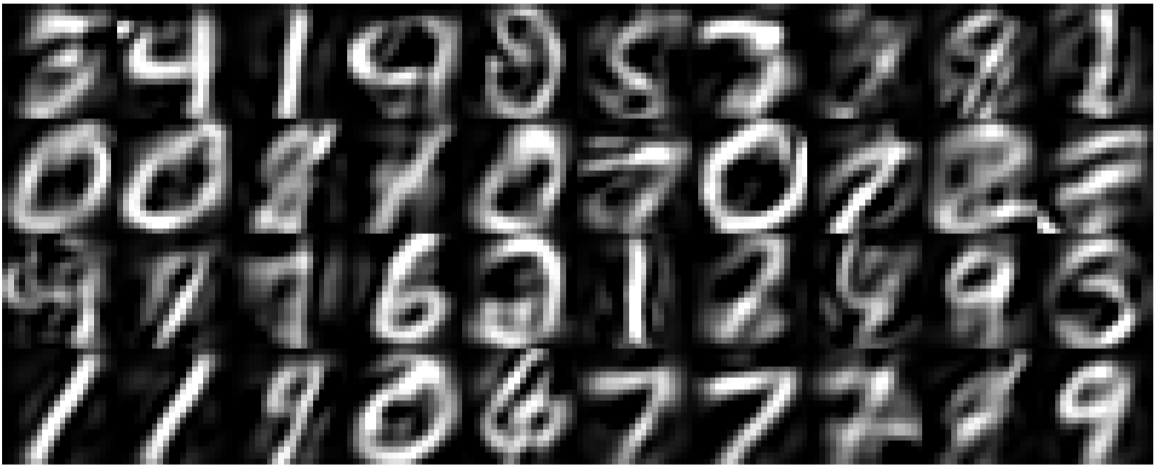}
		\end{minipage}%
	}%
	
	\subfigure[Recovered images via $\ell_4$-based method with top $5$ bases.]{
		\begin{minipage}[t]{0.5\linewidth}
			\centering
			\includegraphics[width=0.8\linewidth]{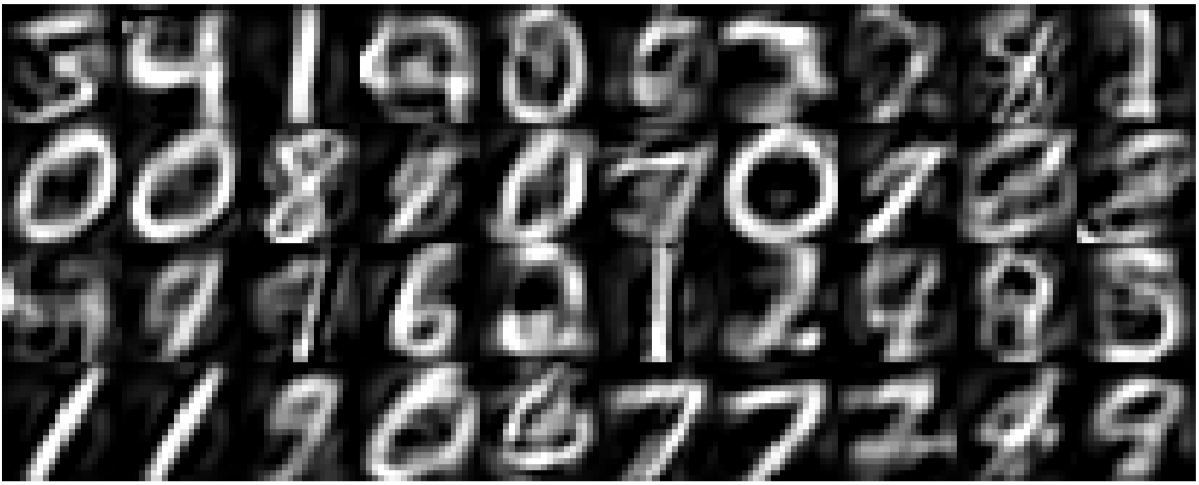}
		\end{minipage}%
	}%
	\subfigure[Recovered images via PCA with top $5$ bases.]{
		\begin{minipage}[t]{0.5\linewidth}
			\centering
			\includegraphics[width=0.8\linewidth]{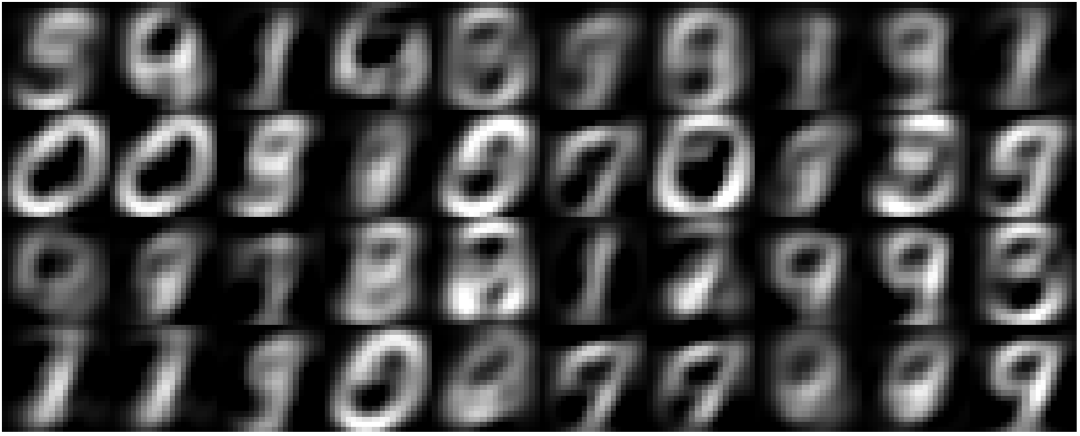}
		\end{minipage}%
	}%
	\centering
	\caption{Recovered images with top $5$ bases.}
	\label{fig:mnist}
\end{figure*}

\section{Correctness of the Global Maximizers}
In this section, we prove Theorem \ref{thm:ex_re} and Theorem \ref{thm:c_obj_n}. We first develop a general theory for the correctness of the global maximizers. Consider the following optimization problem to recover the dictionary
\begin{equation} 
\begin{aligned}
&\underset{\bm{A}}{\text{maximize}}
& & f(\bm{A},\bm{Y}) & & &\text{ subject to } \bm{A} \in \mathbb{O}(n). \notag
\end{aligned}
\end{equation}

The following result shows that if two conditions are satisfied, the global maximizers to this optimization problem are very close to the true dictionary with high probability.

\begin{thm} \label{thm:cor_g}
	Let $\bm{X} \in \mathbb{R}^{n \times r}, x_{i,j} \sim \mathcal{BG}(\theta)$, $\bm{D}_0 \in \mathbb{O}(n)$ be an orthogonal dictionary, and $\bm{Y} = \bm{D}_0\bm{X}$.

	Assume a function $f:\mathbb{R}^{n\times n} \times \mathbb{R}^{n \times r} \rightarrow \mathbb{R}$ that satisfies the following two conditions.
	
	1. (Concentration) Denote $g(\bm{A}) = \mathbb{E}_{\bm{Y}} f(\bm{A},\bm{Y})$, with sample size $r > \Phi(\delta,n)$, and then
	\begin{align} \label{cond:con}
		\sup_{\bm{A} \in \mathbb{O}(n)}|f(\bm{A},\bm{Y}) - g(\bm{A})| \leq \delta
	\end{align}
	holds with high probability.
	
	2. (Sharpness) For all $\bm{A} \in \mathbb{O}(n)$, there exists $\alpha>0$ and $\bm{\Pi} \in \text{SP}(n)$ such that
	\begin{align} \label{cond:shp}
		g(\bm{D}_0^*\bm{\Pi}^*) - g(\bm{A}) \geq \alpha \min_{\bm{\Pi} \in \text{SP}(n)} \|\bm{A} -  \bm{D}_0^*\bm{\Pi}^*\|_F^2.
	\end{align}
	
	Suppose $\hat{\bm{A}}$ is a global maximizer to 
	\begin{equation} \label{eq:g_for}
	\begin{aligned}
	&\underset{\bm{A}}{\text{maximize}}
	& & f(\bm{A},\bm{Y}) & & &\text{ subject to } \bm{A} \in \mathbb{O}(n). 
	\end{aligned}
	\end{equation}
	and $r > \Phi(\delta,n)$, then for any $\delta > 0$, there exists a signed permutation $\bm{\Pi}$, such that
	\begin{align*}
	\left\|\hat{\bm{A}}^* - \bm{D}_0\bm{\Pi} \right\|_F^2 \leq 2\alpha^{-1}\delta
	\end{align*}
	with high probability.
\end{thm}

\begin{proof}
	From concentration, if $r > \Phi(\delta,n)$ we have 
	\begin{align*}
	&\left|f(\hat{\bm{A} },\bm{Y}) - g(\hat{\bm{A}}) \right| \leq \delta, &&\left| f(\bm{D}_0^*,\bm{Y}) - g(\bm{D}_0^*) \right| \leq \delta.
	\end{align*}
	with high probability. This leads to 
	\begin{align} \label{eq:ep_obj_g}
	g(\bm{D}_0^*) - 2\delta \leq g(\hat{\bm{A}}) \leq g(\bm{D}_0^*).
	\end{align}
	
	By sharpness, we have
	\begin{align*}
		\alpha \|\hat{\bm{A}}^* - \bm{D}_0\bm{\Pi}\|_F^2 \leq g(\bm{D}_0^*) - g(\hat{\bm{A}}) \leq 2\delta.
	\end{align*}
	
	This implies 
	\begin{align*}
	\|\hat{\bm{A}}^* - \bm{D}_0\bm{\Pi}\|_F^2 \leq 2\alpha^{-1}\delta
	\end{align*}
	with high probability.
	
\end{proof}

\begin{rem}
	The concentration condition determines the sample complexity of the formulation and ensures that the maximal objective values are close to the expected maximal objective value. The sharpness condition measures how close two variables are if the their objective values are close, which then leads to the closeness in the variable values.
\end{rem}

We check the concentration and sharpness of noiseless and noisy objectives in the following two subsections.

\subsection{Proof of Theorem \ref{thm:ex_re}} \label{sec:noiseless}
In this subsection, we prove Theorem \ref{thm:ex_re} via Theorem \ref{thm:cor_g}. Specifically, the concentration and sharpness conditions are established in Lemma \ref{lem:c_o} and Lemma \ref{lem:shp_o} respectively. Then we show the global maximizers of all $\ell_p$-based formulation are very close to the true dictionary with high probability via Theorem \ref{thm:cor_g}.

We begin with the correctness of the population objective. We first show the correctness over the sphere and then extend it to $\mathbb{O}(n)$.
\begin{lem} (Global optimums over the sphere) \label{lem:go_s} Consider the problem
	\begin{equation}
	\begin{aligned}
	&\underset{\bm{a} \in \mathbb{S}^{n-1}}{\text{maximize}}
	& & \mathbb{E}_{\bm{y}}|\bm{a}^*\bm{y}|^p,
	\end{aligned}		
	\end{equation}
	where $\bm{y} = \bm{D}_0 (\bm{b} \circ \bm{g})$ with $\bm{b} \sim \text{Ber}(\theta)$ and $\bm{g} \sim \mathcal{N}(0,\sigma^2)$, $\gamma_p = \sigma^{p} 2^{p/2} \frac{\Gamma(\frac{p+1}{2})}{\sqrt{\pi}}$ and $\Gamma(\cdot)$ is the Gamma function. The equality holds when $\|\bm{a}^*\bm{D}_0\|_0 = 1$ and the maximal objective value is $\gamma_{p} \theta$.
\end{lem}

\begin{proof}
	We prove this lemma by separating $\bm{b}$ and $\bm{g}$ 
	\begin{equation}
	\begin{aligned}
	&\mathbb{E}_{\bm{y}}|\bm{a}^*\bm{y}|^p =  \mathbb{E}_{\bm{b},\bm{g}}|\bm{a}^*\bm{D}_0(\bm{b} \circ \bm{g})|^p = \mathbb{E}_{\bm{b},\bm{g}}|(\bm{a}^*\bm{D}_0 \circ \bm{b})  \bm{g}|^p = \mathbb{E}_{\bm{g},\Omega}|(\bm{a}^*\bm{D}_0)_{\Omega} \bm{g}|^p \overset{(a)}{=} \gamma_p \mathbb{E}_{\Omega}\|(\bm{a}^*\bm{D}_0)_{\Omega}\|_2^p \\
	\overset{(b)}{\leq} & \gamma_p \mathbb{E}_{\Omega}\|(\bm{a}^*\bm{D}_0)_{\Omega}\|_2^2 = \gamma_p \theta, \notag
	\end{aligned}
	\end{equation}
	where $\Omega$ denotes the support of $\bm{b}$ and inequality (a) follows Lemma \ref{lem:gm}. The inequality in (b) is because $\|(\bm{a}^*\bm{D}_0)_{\Omega}\|_2 \leq 1$ and the equality holds only if $\|\bm{a}^*\bm{D}_0\|_0 = 1$.
\end{proof}

Extending Lemma \ref{lem:go_s} to $\mathbb{O}(n)$ results in the following lemma, which completes the proof for the correctness of the population objective. 
\begin{lem} \label{lem:go_st}
	The only global maximizers to 
	\begin{equation} 
	\begin{aligned}
	&\underset{\bm{A}}{\text{maximize}}
	& & \mathbb{E}_{\bm{X}_0}\|\bm{A}\bm{Y}\|_p^p & & &\text{ subject to } \bm{A} \in \mathbb{O}(n) \notag
	\end{aligned}
	\end{equation}
	are $\bm{A}^* = \bm{D}_0 \bm{P}$, where $\bm{P}$ is any signed permutation matrix.
\end{lem}

\begin{proof} 
	We consider 
	\begin{equation} 
	\begin{aligned}
	&\underset{\bm{A}}{\text{maximize}}
	& & \mathbb{E}_{\bm{Y}}\|\bm{A}\bm{Y}\|_p^p & & &\text{ subject to } \bm{A} \in \mathbb{O}(n). \notag
	\end{aligned}
	\end{equation}
	
	Denoting $\bm{A} = [\bm{a}_1^*;\cdots;\bm{a}_m^*]$ where $\bm{a}_i \in \mathbb{R}^n$ and $\bm{Y} = [\bm{y}_1,\cdots,\bm{y}_r]$, we have
	\begin{equation}
	\begin{aligned}
	\mathbb{E}_{\bm{Y}} \|\bm{A}\bm{Y}\|_p^p = \sum_{i=1}^n \sum_{j=1}^r \mathbb{E}_{\bm{y}_j}|\bm{a}_i^*\bm{y}_j|_p^p = r \sum_{i=1}^n \mathbb{E}_{\bm{y}_j}|\bm{a}_i^*\bm{y}_j|_p^p \overset{(a)}{\leq} rn \gamma_p \theta.
	\end{aligned}
	\end{equation}
	Inequality (a) follows Lemma \ref{lem:go_s} and the equality holds only if $\|\bm{a}_i\bm{D}_0\|_0=1,\forall i$ and $\bm{A} \in \text{St}(n,m)$. Thus, the global maximum is achieved only if $\bm{A}^* = \bm{D}_0 \bm{P}$, where $\bm{P}$ is any signed permutation matrix.
\end{proof}

We then establish the concentration of the empirical objective. We first show a general heavy-tailed concentration bound over the Stiefel manifold, and then apply it to obtain the required sample complexity for concentration. 

\begin{lem}
	(Concentration on Stiefel manifold) \label{lem:c_st}
	Let $\bm{z}_1,\bm{z}_2,\cdots,\bm{z}_r \in \mathbb{R}^{n_1}$ be i.i.d. centered subgaussian random vectors, with $\bm{z}_i \equiv_d \bm{z} (1\leq i \leq r)$ such that 
	\begin{equation}
	\begin{aligned}
	&\mathbb{E}[z_i] = 0, &&\mathbb{P}(|z_i| > t) \leq 2\exp\left(-\frac{t^2}{2\sigma^2} \right). \notag
	\end{aligned}
	\end{equation}
	For fixed $\bm{Q} \in \text{St}(n,m)$, we define a function $f_{\bm{Q}}: \mathbb{R}^{n_1} \rightarrow \mathbb{R}^{d_1}$, such that 
	
	1. $f_{\bm{Q}}(\bm{z})$ is a heavy tailed process of $\bm{z}$, in the sense of 
	\begin{align} \label{eq:f1}
	\mathbb{P}(\|f_{\bm{Q}}(\bm{z})\|_2 \geq t)\leq 2\exp(-Ct^{2/p}).
	\end{align}
	
	2. The expectation $\mathbb{E}[f_{\bm{Q}}(\bm{z})]$ is bounded and $L_f$-Lipschitz, i.e.,
	\begin{equation}
	\begin{aligned} \label{eq:f2}
	&\|\mathbb{E}[f_{\bm{Q}}(\bm{z})]\| \leq B_f, &\text{and} &&\|\mathbb{E}[f_{\bm{Q}_1}(\bm{z})] - \mathbb{E}[f_{\bm{Q}_2}(\bm{z})]\| \leq L_f \|\bm{Q}_1 - \bm{Q}_2\|, \forall \bm{Q}_1, \bm{Q}_2 \in \text{St}(n,m).
	\end{aligned}
	\end{equation}
	
	3. Let $\bar{\bm{z}}$ be a truncated vector of $\bm{z}$, such that 
	\begin{equation}
	\begin{aligned} \label{eq:f3}
	&\bm{z} = \bar{\bm{z}} + \hat{\bm{z}}, &&\bar{z_i} = \left\{
	\begin{aligned}
	&z_i, &&\text{if } |z_i| \leq B \\
	&0, &&\text{otherwise}.
	\end{aligned}
	\right.
	\end{aligned} 
	\end{equation}
	with $B = 2\sigma\sqrt{\log(n_1r)}$. We further assume that 
	\begin{align} \label{eq:f4}
	\|f_{\bm{Q}}(\bar{\bm{z}})\| \leq &R_1, \quad \mathbb{E}\left[\|f_{\bm{Q}}(\bar{\bm{z}})\|^2 \right] \leq R_2 \notag\\
	\|f_{\bm{Q}_1}(\bar{\bm{z}}) - f_{\bm{Q}_2}(\bar{\bm{z}})\| &\leq \bar{L}_f\|\bm{Q}_1-\bm{Q}_2\|, \forall \bm{Q}_1, \bm{Q}_2 \in \text{St}(n,m).
	\end{align}
	Then for any $t>0$, we have
	\begin{equation}
	\begin{aligned}
	&\mathbb{P}\left(\underset{\bm{Q} \in \text{St}(n,m)}{\sup} \left\|\frac{1}{mr} \sum_{i=1}^r f_{\bm{Q}} (\bm{z}_i)   - \mathbb{E}[f_{\bm{Q}} (\bm{z}_i) ]\right\| > t \right) \\
	\leq &(n_1r)^{-1} + \exp \left(-\min \left\{ \frac{rm^2t^2}{64R_2}, \frac{3mrt}{32R_1}\right\} + nm\log\left(\frac{12(L_f + \bar{L}_f)}{t}\right) + \log(d_1)\right).
	\end{aligned}
	\end{equation}
	In other words, for any given $\delta$, whenever 
	$$r \geq Cn/\delta\log\left(\frac{(L_f + \bar{L}_f)d_1}{\delta}\right)\max\left\{R_2/(m\delta),R_1   \right\} + \frac{B_f}{\sqrt{n_1} \delta},$$
	we have
	$$\mathbb{P}\left(\underset{\bm{Q} \in \text{St}(n,m)}{\sup} \left\|\frac{1}{mr} \sum_{i=1}^r f_{\bm{Q}} (\bm{z}_i)   - \mathbb{E}[f_{\bm{Q}} (\bm{z}_i) ]\right\| > \delta \right) < (n_1r)^{-1} + (mn)^{-c\log\left(\frac{12(L_f + \bar{L}_f)}{\delta}\right)}.$$

\end{lem}

\begin{proof}
	The proof is heavily based on the concentration of random vectors over the sphere (Theorem F.1 in \cite{qu2019geometric}). 
	
	We employ truncation to deal with the heavy-tailed phenomenon and thus the bounds for bounded random variables \cite{vershynin2018high} can be applied here. The truncation level is set as $B=2\sigma\sqrt{\log(n_1r)}$. We first separate the concentration into three parts
	
	\begin{align*}
	&\mathbb{P}\left( \sup_{\bm{Q} \in \text{St}(n,m)} \left\|\frac{1}{r}\sum_{i=1}^r f_{\bm{Q} }(\bm{z}_i)-\mathbb{E}f_{\bm{Q}}(\bm{z}) \right\| \geq t \right) \\
	\leq & \underbrace{\mathbb{P} \left(\sup_{\bm{Q} \in \text{St}(n,m)}  \left\|\frac{1}{r} \sum_{i=1}^r f_{\bm{Q}}(\bar{\bm{z}}_i) - \mathbb{E}f_{\bm{Q}}(\bar{\bm{z}})\right\| \geq \frac{t}{2}  \right) }_{\mathcal{T}_1} + \underbrace{\mathbb{P} \left(\sup_{\bm{Q} \in \text{St}(n,m)}  \left\|\mathbb{E} f_{\bm{Q}}(\bar{\bm{z}}) - \mathbb{E}f_{\bm{Q}}(\bm{z})\right\| \geq \frac{t}{2}  \right) }_{\mathcal{T}_2} + \underbrace{\mathbb{P} \left(\max_{1\leq i \leq r} \|\bm{z}_i\|_{\infty} \leq B \right)}_{\mathcal{T}_3}.
	\end{align*} 
	
	Denote the $\epsilon$-net on Stiefel manifold as $N(\epsilon)$.
	\paragraph{Bound $\mathcal{T}_3$}
	\begin{align*}
		\mathcal{T}_3 = \mathbb{P}\left(\max_{1\leq i \leq r} \|\bm{z}_i\|_{\infty} \geq B   \right) \leq n_1r\mathbb{P}(|z_{i,j}|\geq B) \leq \exp\left(-\frac{B^2}{2\sigma^2} + \log(n_1r)   \right) = (n_1r)^{-1}.
	\end{align*}
	
	\paragraph{Bound $\mathcal{T}_2$}
	Note that
	\begin{align*}
	\left\|\mathbb{E} f_{\bm{Q}}(\bm{z}) - \mathbb{E}f_{\bm{Q}}(\bm{z})\right\| \leq \|\mathbb{E} f_{\bm{Q}}(\bm{z}) \circ \bm{1}_{\bm{z} \neq \bar{\bm{z}}}\|_2 \leq \|\mathbb{E} f_{\bm{Q}}(\bm{z}) \|_2 \|\bm{1}_{\bm{z} \neq \bar{\bm{z}}}\|_2 \leq B_f r^{-1}n_1^{-\frac{1}{2}}.
	\end{align*}
	
	Thus, we have $\mathcal{T}_2 = 0$ if $r\geq 2B_f t^{-1}n_1^{-\frac{1}{2}}$.
	
	\paragraph{Bound $\mathcal{T}_1$}
	To bound $\mathcal{T}_1$, we first derive a bound for fixed $\bm{Q} \in \text{St}(n,m)$ and then take a union bound over $N(\epsilon)$. By Bernstein inequality for bounded random variables [Theorem 2.8.4 in \cite{vershynin2018high}], we have
	
	\begin{align*}
	\mathbb{P} \left(\left\|\mathbb{E} f_{\bm{Q}}(\bar{\bm{z}}_i) - \mathbb{E}f_{\bm{Q}}(\bm{z})\right\| \geq \frac{t}{2}  \right) \leq d_1 \exp\left(-\frac{rt^2}{8R_2+8R_1t/3}   \right).
	\end{align*}
	
	\paragraph{Covering over the Stiefel manifold} We know that 
	\begin{align*}
		\forall \bm{Q} \in \text{St}(n,m), \quad \exists \bm{Q}' \in N(\epsilon), \quad  \text{such that} \quad \|\bm{Q}-\bm{Q}'\| \leq \epsilon, \quad \text{and } |N(\epsilon)| \leq \left(\frac{6}{\epsilon}  \right)^{mn}.
	\end{align*}
	We have 
	\begin{align*}
		&\sup_{\bm{Q} \in \text{St}(n,m)}  \left\|\frac{1}{r} \sum_{i=1}^r f_{\bm{Q}}(\bar{\bm{z}}_i) - \mathbb{E}f_{\bm{Q}}(\bm{z})\right\| = \sup_{\bm{Q} \in \text{St}(n,m), \|\bm{e}\| \leq \epsilon }  \left\|\frac{1}{r} \sum_{i=1}^rf_{\bm{Q}+\bm{e}}(\bar{\bm{z}}_i) - \mathbb{E}f_{\bm{Q}+\bm{e}}(\bm{z})\right\| \\
		&\leq \sup_{\bm{Q}' \in N(\epsilon)}  \left\|\frac{1}{r} \sum_{i=1}^r f_{\bm{Q}'}(\bar{\bm{z}}_i) - \mathbb{E}f_{\bm{Q'}}(\bm{z})\right\| + \sup_{\bm{Q}' \in N(\epsilon), \|\bm{e}\| \leq \epsilon }  \left\|\frac{1}{r} \sum_{i=1}^r f_{\bm{Q}'+\bm{e}}(\bar{\bm{z}}_i) - \frac{1}{r} \sum_{i=1}^r f_{\bm{Q}'}(\bar{\bm{z}}_i)\right\|\\ &+ \sup_{\bm{Q}' \in N(\epsilon), \|\bm{e}\| \leq \epsilon }  \left\|\mathbb{E}f_{\bm{Q}'+\bm{e}}(\bm{z}) - \mathbb{E}f_{\bm{Q}'}(\bm{z})\right\|.
	\end{align*}
	
	Due to the Lipschitz condition in Equation (\ref{eq:f1}) and Equation (\ref{eq:f4}), we have 
	\begin{align*}
		\left\|\mathbb{E}f_{\bm{Q}'+\bm{e}}(\bm{z}) - \mathbb{E}f_{\bm{Q}'}(\bm{z})\right\| &\leq L_f\|\bm{e}\|, \\
		\left\|\frac{1}{r} \sum_{i=1}^r f_{\bm{Q}'+\bm{e}}(\bar{\bm{z}}_i) - \frac{1}{r} \sum_{i=1}^r f_{\bm{Q}'}(\bar{\bm{z}}_i)\right\| &\leq \| f_{\bm{Q}'+\bm{e}}(\bar{\bm{z}}) - f_{\bm{Q}'}(\bar{\bm{z}})\| \leq \bar{L}_f \|\bm{e}\|.
	\end{align*}
	
	This implies 
	\begin{align*} 
		\sup_{\bm{Q} \in \text{St}(n,m)}  \left\|\frac{1}{r} \sum_{i=1}^r f_{\bm{Q}}(\bar{\bm{z}}_i) - \mathbb{E}f_{\bm{Q}}(\bm{z})\right\| \leq  \sup_{\bm{Q}' \in N(\epsilon)}  \left\|\frac{1}{r} \sum_{i=1}^r f_{\bm{Q}'}(\bar{\bm{z}}_i) - \mathbb{E}f_{\bm{Q'}}(\bm{z})\right\| + (L_f+\bar{L}_f)\epsilon.
	\end{align*}
	
	Choose $\epsilon \leq \frac{t}{2(L_f+\bar{L}_f)}$, and we have
	\begin{align*}
	\mathcal{T}_1 \leq& \mathbb{P}\left(\sup_{\bm{Q}' \in N(\epsilon)}  \left\|\frac{1}{r} \sum_{i=1}^r f_{\bm{Q}'}(\bar{\bm{z}}_i) - \mathbb{E}f_{\bm{Q}'}(\bm{z})\right\| \geq t - (L_f+\bar{L}_f)\epsilon   \right) \\
	\leq& \mathbb{P}\left(\sup_{\bm{Q}' \in N(\epsilon)}  \left\|\frac{1}{r} \sum_{i=1}^r f_{\bm{Q}'}(\bar{\bm{z}}_i) - \mathbb{E}f_{\bm{Q}'}(\bm{z})\right\| \geq t/2   \right) \\
	\overset{(a)}{\leq}& |N(\epsilon)|  \mathbb{P}\left(\left\|\frac{1}{r} \sum_{i=1}^r f_{\bm{Q}}(\bar{\bm{z}}_i) - \mathbb{E}f_{\bm{Q}}(\bm{z})\right\| \geq t/2\right) \\
	\overset{(b)}{\leq}& \left(\frac{6}{\epsilon}\right)^{mn} d_1\exp\left(-\frac{pt^2}{32R_2+16R_1t/3}   \right) \\
	\leq& \exp \left(-\min \left\{ \frac{rt^2}{64R_2}, \frac{3rt}{32R_1}\right\} + nm\log\left(\frac{12(L_f + \bar{L}_f)}{t}\right) + \log(d_1)\right).
	\end{align*}
	Inequality (a) can be obtained by taking a union bound over the $\epsilon$-net and inequality (b) follows Lemma \ref{lem:e_st}.
	
	We finish the proof by taking $t = m\delta$.
	
\end{proof}

We then use Lemma \ref{lem:c_st} to establish concentration for $\ell_p$ objectives.

\begin{lem} (Concentration) \label{lem:c_o}
	Suppose $\bm{X} \in \mathbb{R}^{n \times r}$ follows $\mathcal{BG}(\theta)$. For any given $\theta \in (0,1)$ and $\delta > 0$, whenever
	$$r \geq C \delta^{-2} n\log ( n/\delta) (\theta n \log^2 n)^{\frac{p}{2}},$$
	we have 
	$$\underset{\bm{A} \in \mathbb{O}(n) }{\sup} \frac{1}{nr}\left| \|\bm{A}\bm{Y}\|_p^p - \mathbb{E}(\|\bm{A}\bm{Y}\|_p^p) \right| \leq \delta,$$
	with probability at least $1-r^{-1}$.
\end{lem}
\begin{proof}
	Note that 
	\begin{align*}
	\|\bm{A}\bm{D}_0\bm{X}_0\|_p^p = \sum_{i=1}^r \|\bm{A}\bm{D}_0\bm{x}_i\|_p^p.
	\end{align*}
	To use Lemma \ref{lem:c_st}, we define $\bm{Q}=\bm{A}\bm{D}_0$, $\bm{z} = \bm{x}$ and $f_{\bm{Q}}(\bm{z}) = \|\bm{Q}\bm{z}\|_p^p$.
	
	\textbf{Bound $R_2$}
	
	Denote $\bm{Q} \in \mathbb{O}(n)$ and $\bm{Q} = [\bm{q}_1^*;\cdots;\bm{q}_n^*]$.
	\begin{align} \label{eq:r2_1}
	&\mathbb{E}\|f_{\bm{Q}}(\bar{\bm{z}})\|^2 \leq \mathbb{E}|f_{\bm{Q}}(\bm{z})|^2 = \mathbb{E}\|\bm{Q}\bm{z}\|_p^{2p} = \mathbb{E}(\|\bm{Q}\bm{z}\|_p^p)^2 = \mathbb{E} \left(\sum_{i=1}^n|\bm{q}_i^*\bm{z} |^p  \right)^2 = \sum_{i=1}^n\mathbb{E} |\bm{q}_i\bm{z}|^{2p} + \sum_{i=1}^n \sum_{j=1,j\neq i}^n \mathbb{E} |\bm{q}_i\bm{z}|^p\mathbb{E}|\bm{q}_j\bm{z}|^p \\ \notag
	\overset{(a)}{\leq} & \gamma_{2p}\theta n+ \theta^2\gamma_{p}^2(n^2-n) = R_2.
	\end{align}
	
	\textbf{Bound $R_1$}
	By Cauchy inequality, we have
	\begin{align} \label{eq:r1_1}
	\|\bm{Q}\bar{\bm{z}}\|_p^p \leq \|\bm{Q}\bar{\bm{z}}\|_2^p \leq \|\bm{Q}\|_2^p \|\bar{\bm{z}}\|_2^p = (\|\bar{\bm{z}}\|_2^2)^{p/2} \leq (B^2 \|\bar{\bm{z}}\|_0)^{p/2} \overset{(a)}{\leq} (B^2 4 \theta n \log(r))^{p/2}.
	\end{align}
	with probability at least $1-\exp(-\theta n)$ and (a) follows Lemma \ref{lem:ber}. 
	
	\textbf{Bound $L_f, \bar{L}_f$}
	Note that sample complexity bound in Lemma \ref{lem:c_st} is propositional to $\log(L_f+\bar{L}_f)$ and $\log(L_f+\bar{L}_f) = \Theta(\log n)$ as long as $L_f$ and $\bar{L}_f$ is in the polynomial of $n$. Thus, it is sufficient to bound them in the polynomials of $n$.
	By calculation, we have
	\begin{align} \label{eq:Lf1}
	\bar{L}_f \leq c_1 p n^{p+1} B^p, \quad L_f \leq c_2 n p.
	\end{align}
	
	\textbf{Bound $B_f$}
	
	Apply Lemma \ref{lem:go_s}, we have 
	\begin{align} \label{eq:Bf_1}
	B_f = \mathbb{E}\|\bm{Q}\bm{z}\|_p^p \leq \sum_{i=1}^n\mathbb{E}|\bm{q}_i\bm{z}|^p \leq n\theta\gamma_p.
	\end{align}
	
	We finish the proof by substituting (\ref{eq:r2_1}-\ref{eq:Bf_1}) into Lemma \ref{lem:c_st}.
\end{proof}

In the following two lemmas, we show the sharpness of $\ell_p$ objectives.

\begin{lem} \label{lem:shp}
	Suppose $\bm{q} \in \mathbb{S}^{n-1}, r = 0.5\min_{1\leq i\leq n}\{\|\bm{q}-\bm{e}_i\|_2^2,\|\bm{q}+\bm{e}_i\|_2^2\}$,  then we have
	\begin{align} \label{eq:shp_s}
	r \leq \frac{C_p}{\theta(1-\theta)}(\theta-\mathbb{E}_{\Omega}\|\bm{q}_{\Omega}\|_2^p),
	\end{align}
	where $C_p = (1- 2(0.5)^{\frac{p}{2}})^{-1}$.
\end{lem}

\begin{proof}

	Without loss of generality, we assume $r = 0.5\min\{\|\bm{q}-\bm{e}_n\|_2^2,\|\bm{q}+\bm{e}_n\|_2^2\}$ and thus $r = |1-q_n|^2 = \epsilon^2$. 
	
	We first show (\ref{eq:shp_s}) holds when $\epsilon \leq \frac{1}{\sqrt{2} }$. Let $\Omega' = \Omega\backslash\{n\}$, we have
	\begin{align*}
	&\mathbb{E}_{\Omega}\|\bm{q}_{\Omega}\|_2^p = \mathbb{E}_{\Omega}(\|\bm{q}_{\Omega}\|_2^2)^{\frac{p}{2}} = \theta \mathbb{E}_{\Omega'} (\|\bm{q}_{\Omega'}\|^2+1-\epsilon^2)^{p/2} + (1-\theta )\mathbb{E}_{\Omega'} (\|\bm{q}_{\Omega'}\|^2)^{p/2} \\
	\overset{(a)}{\leq} &\theta(1-\theta)(\epsilon^p + (1-\epsilon^2)^{p/2}) + (1-\theta)^2 \cdot 0 + \theta^2 \cdot 1. 
	\end{align*}
	Equality in (a) holds only if $\|\bm{q}\|_0 = 2$.
	
	Define $f(\epsilon) = 1 - \epsilon^p - (1-\epsilon^2)^{p/2}$ and $g(\epsilon) = f(\epsilon) - 2f(\frac{1}{\sqrt{2}})\epsilon^2$. For $0 \leq \epsilon \leq \frac{1}{\sqrt{2}}$, $g(\epsilon) \geq 0$ and the equality holds only if $\epsilon = 0$ or $\epsilon = \frac{1}{\sqrt{2}}$.
	
	Thus, we have 
	\begin{align*}
		&\inf_{\bm{q} \in \mathbb{S}^{n-1}}\theta - \mathbb{E}_{\Omega}\|\bm{q}_{\Omega}\|_2^p = \theta - \sup_{\bm{q} \in \mathbb{S}^{n-1}} \mathbb{E}_{\Omega}\|\bm{q}_{\Omega}\|_2^p
		= \theta(1-\theta)f(\epsilon) \geq 2\theta(1-\theta) f(\frac{1}{\sqrt{2}})\epsilon^2 = 2\theta(1-\theta) f(\frac{1}{\sqrt{2}}) r,
	\end{align*}
	which implies 
	\begin{align*}
		r \leq \frac{C_p}{2\theta(1-\theta)}(\theta-\mathbb{E}\|\bm{q}_{\Omega}\|_2^p), 
	\end{align*}
	for $\epsilon \leq \frac{1}{\sqrt{2}}$, where $C_p = f^{-1}(\frac{1}{\sqrt{2}})$.
	
	Next, we show that (\ref{eq:shp_s}) holds when $\epsilon > \frac{1}{\sqrt{2}}$. As $r\leq 1$ always holds, we can take the smallest $c$ such that
	\begin{align} \label{eq:ep_g}
		1 \leq \frac{c}{\theta(1-\theta)} (\theta - \mathbb{E}_{\Omega}\|\bm{q}_{\Omega}\|_2^p),
	\end{align} 
	holds for $\epsilon > \frac{1}{\sqrt{2}}$.
	
	Denote $\bm{w} = \bm{q}_{1:n-1}$ and $g(\bm{w}) = \mathbb{E}_{\Omega}\|[\bm{w},\sqrt{1-\|\bm{w}\|_2^2}]_{\Omega}\|_2^p$. Then, we have
	\begin{align*}
		\theta - \mathbb{E}_{\Omega}\|\bm{q}_{\Omega}\|_2^p = \theta - g(\bm{w}) \geq \theta -  g\left(\frac{\bm{w}}{\sqrt{2}\epsilon}\right) \geq 2\theta(1-\theta)f(\frac{1}{\sqrt{2}})\left(\frac{1}{\sqrt{2}}\right)^2 \geq f(\frac{1}{\sqrt{2}}) \theta(1-\theta).
	\end{align*}
	
	Thus, we take $c=C_p$ in (\ref{eq:ep_g}) to finish the proof.

\end{proof}

\begin{lem} (Sharpness)\label{lem:shp_o}
	Suppose $\bm{Q} \in \mathbb{O}(n)$, and then $\exists \bm{P} \in \text{SP}(n)$ such that
	\begin{align}\label{eq:shp}
		\theta - \frac{1}{nr\gamma_p}  \mathbb{E}_{\bm{Y}} \|\bm{Q}\bm{Y}\|_p^p \geq \frac{\theta(1-\theta)}{2nC_p} \|\bm{Q} - \bm{P}\|_2^2,
	\end{align}
	where $C_p = (1 - 2(0.5)^{\frac{p}{2}})^{-1}$.
\end{lem}

\begin{proof}
	Denote $\bm{Q} = [\bm{q}_1;\cdots;\bm{q}_n]$, we have
	\begin{align*}
		&\theta - \frac{1}{nr\gamma_p} \mathbb{E}_{\bm{Y}} \|\bm{Q}\bm{Y}\|_p^p = \theta - \frac{1}{n} \sum_{i=1}^n\mathbb{E}_{\Omega} \|\bm{q}_{i,\Omega}\|_2^p = \frac{1}{n}\sum_{i=1}^n(\theta -  \mathbb{E}_{\Omega} \|\bm{q}_{i,\Omega}\|_2^p) \overset{(a)}{\geq} \frac{\theta(1-\theta)}{2nC_p} \sum_{i=1}^n \min_{1\leq j\leq n} \{\|\bm{q}_i-\bm{e}_j\|_2^2,\|\bm{q}_i+\bm{e}_j\|_2^2\} \\
		\overset{(b)}{=} & \frac{\theta(1-\theta)}{2nC_p} \|\bm{Q}-\bm{P}\|_2^2
	\end{align*}
	where (a) follows Lemma \ref{lem:shp} and (b) follows (A.36) in \cite{zhai2019complete}.
\end{proof}

We give a proof for Theorem \ref{thm:ex_re} through Theorem \ref{thm:cor_g}, assisted by Lemma \ref{lem:c_o} and Lemma \ref{lem:shp_o}. 

\begin{thm} 
	Let $\bm{X} \in \mathbb{R}^{n \times r}, x_{i,j} \sim \mathcal{BG}(\theta)$ with $\theta \in (0,1)$, $\bm{D}_0 \in \mathbb{O}(n)$ is an orthogonal dictionary, and $\bm{Y} = \bm{D}_0\bm{X}$. Suppose $\hat{\bm{A}}$ is a global maximizer to 
	\begin{equation} 
	\begin{aligned}
	&\underset{\bm{A}}{\text{maximize}}
	& & \|\bm{A}\bm{Y}\|_p^p \text{ subject to } \bm{A} \in \mathbb{O}(n). \notag
	\end{aligned}
	\end{equation}
	Provided that the sample size $r = \Omega\left(\theta \delta^{-2} n\log (n/\delta) (n \log^2 n)^{\frac{p}{2}}\right)$, then for  $\delta > 0$, there exists a signed permutation $\bm{\Pi}$, such that
	\begin{align*}
	\frac{1}{n}\left\|\hat{\bm{A}}^* - \bm{D}_0\bm{\Pi} \right\|_F^2 \leq C_{\theta} \delta,
	\end{align*}
	with probability at least $1-r^{-1}$ and $C_{\theta}$ is a constant that depends on $\theta$.
\end{thm}

\begin{proof}
	To use Theorem \ref{thm:cor_g}, we have to check the concentration of empirical objective and sharpness of the population objective.
	
	\paragraph{Concentration} From Lemma \ref{lem:c_o}, for any $\delta > 0$, whenever $r > \Phi(\delta,n) = \Omega\left(\theta \delta^{-2} n\log (n/\delta) (n \log^2 n)^{\frac{p}{2}}\right)$, we have
	$$\underset{\bm{A} \in \mathbb{O}(n) }{\sup} \frac{1}{nr}\left| \|\bm{A}\bm{Y}\|_p^p - \mathbb{E}(\|\bm{A}\bm{Y}\|_p^p) \right| \leq \delta,$$
	with probability at least $1-r^{-1}$.
	
	\paragraph{Sharpness} 
	From Lemma \ref{lem:shp_o}, $\forall \bm{A} \in \mathbb{O}(n)$, we
	\begin{align*}
		\frac{1}{nr}(\mathbb{E}(\|\bm{D}_0^*\bm{Y}\|_p^p) - \mathbb{E}(\|\bm{A}\bm{Y}\|_p^p)) = \gamma_p\theta - \frac{1}{nr}\mathbb{E}(\|\bm{A}\bm{Y}\|_p^p) \geq \frac{\gamma_p\theta(1-\theta)}{2nC_p} \|\bm{A} - \bm{P}\|_2^2
	\end{align*}
	Let $C_{\theta} = \frac{4C_p}{\theta(1-\theta)\gamma_{p}}$, by Theorem \ref{thm:cor_g}, we have 
	\begin{align*}
	\frac{1}{n} \|\hat{\bm{A}}^* - \bm{D}_0\bm{\Pi}\|_F^2 \leq C_{\theta} \delta,
	\end{align*}
	whenever $p\geq \Omega\left(\theta \delta^{-2} n\log (n/\delta) (n \log^2 n)^{\frac{p}{2}}\right)$.
\end{proof}

\subsection{Proof of Theorem \ref{thm:c_obj_n}} \label{sec:noisy}
In this subsection, we prove Theorem \ref{thm:c_obj_n} via Theorem \ref{thm:cor_g}. Specifically, the concentration and sharpness conditions are established in Lemma \ref{lem:con_obj_n} and Lemma \ref{lem:shp_o_n} respectively. Then we show the global maximizers of all $\ell_p$-based formulation are very close to the true dictionary with high probability via Theorem \ref{thm:cor_g}.

\begin{lem} (Robustness under Gaussian noise) \label{lem:go_st_n}
	Let $\bm{X} \in \mathbb{R}^{n \times r}, x_{i,j} \sim \mathcal{BG}(\theta)$, $\bm{D}_0 \in \mathbb{O}(n)$ is an orthogonal dictionary, $\bm{Y}_N = \bm{D}_0\bm{X} + \bm{G}$, and $\bm{G} \in \mathbb{R}^{n \times r}$ with $G_{i,j} \sim \mathcal{N}(0,\eta^2)$.
	The only global maximizers to 
	\begin{equation}
	\begin{aligned}
	&\underset{\bm{A}}{\text{maximize}}
	& & \mathbb{E}_{\bm{X}_0,\bm{G}}\|\bm{A}\bm{Y}_N\|_p^p \text{ subject to } \bm{A} \in \mathbb{O}(n), \notag
	\end{aligned}
	\end{equation}
	are $\bm{A}^* = \bm{D}_0 \bm{\Pi}$, where $\bm{\Pi}$ is any signed permutation matrix.  
\end{lem}

\begin{proof} 
	Consider
	\begin{equation} 
	\begin{aligned}
	&\underset{\bm{A}}{\text{maximize}}
	& & \mathbb{E}_{\bm{Y}}\|\bm{A}\bm{Y}\|_p^p \quad \text{ subject to } \bm{A} \in \mathbb{O}(n). \notag
	\end{aligned}
	\end{equation}
	
	Denote $\bm{A} = [\bm{a}_1^*;\cdots;\bm{a}_m^*]$ where $\bm{a}_i \in \mathbb{R}^n$, $\bm{Y} = [\bm{y}_1,\cdots,\bm{y}_r]$, and $\bm{G} = [\bm{g}_1,\cdots,\bm{g}_m]$.
	\begin{equation}
	\begin{aligned}
	&\mathbb{E}_{\bm{Y}_N} \|\bm{A}\bm{Y}_N\|_p^p = \mathbb{E}_{\bm{Y},\bm{G}} \|\bm{A}(\bm{Y}+\bm{G}) \|_p^p = \mathbb{E}_{\bm{Y},\bm{G}} \|\bm{A}(\bm{Y}+\bm{G}) \|_p^p \\ 
	= &\sum_{i=1}^n \sum_{j=1}^r \mathbb{E}_{\bm{y}_j,\bm{g}_j}|\bm{a}_i^*\bm{y}_j+\bm{a}_i^*\bm{g}_j|_p^p = r \sum_{i=1}^n \mathbb{E}_{\Omega}(\|(\bm{a}_{i}\bm{D}_0)_{\Omega}\|_2^2 + \eta^2)^{\frac{p}{2}}.
	\end{aligned}
	\end{equation}
	This term achieves its maxima only if $\|\bm{a}_i\bm{D}_0\|_0=1,\forall i$ and $\bm{A} \in \mathbb{O}(n)$. Thus, the global maximum is achieved only if $\bm{A}^* = \bm{D}_0 \bm{P}$, where $\bm{P}$ is any signed permutation matrix.  
\end{proof}

\begin{lem} (Concentration)\label{lem:con_obj_n} Suppose $\bm{X} \in \mathbb{R}^{n \times r}$ follows $\mathcal{BG}(\theta)$. For any given $\theta \in (0,1)$, whenever
	$$r \geq C \delta^{-2}n\log (n/\delta)((1+\eta^2) n \log n)^{\frac{p}{2}},$$ 
	we have 
	$$\underset{\bm{A} \in \mathbb{O}(n) }{\sup} \frac{1}{nr}\left| \|\bm{A}\bm{Y}_N\|_p^p - \mathbb{E}(\|\bm{A}\bm{Y}_N\|_p^p) \right| \leq \delta,$$
	with probability at least $1-r^{-1}$.
\end{lem}

\begin{proof}
	Similar to the proof for Lemma \ref{lem:c_o}, we use Lemma \ref{lem:c_st} to prove this lemma. Note that 
	\begin{align*}
	\|\bm{A}\bm{Y}_N\|_p^p = \|\bm{A}\bm{D}_0(\bm{X}_0+\bm{D}_0^*\bm{G})\|_p^p = \sum_{i=1}^r \|\bm{A}\bm{D}_0(\bm{x}_i+\bm{g}_i)\|_p^p,
	\end{align*}
	and we can define $\bm{Q}=\bm{A}\bm{D}_0$ and $f_{\bm{Q}}(\bm{z}) = \|\bm{Q}\bm{z}\|_p^p$ and to use Lemma \ref{lem:c_st}.
	
	\textbf{Bound $R_2$}
	
	Denote $\bm{Q} \in \mathbb{O}(n)$ and $\bm{Q} = [\bm{q}_1^*;\cdots;\bm{q}_n^*]$.
	\begin{align} 
	&\mathbb{E}\|f_{\bm{Q}}(\bar{\bm{z}})\|^2 \leq \mathbb{E}|f_{\bm{Q}}(\bm{z})|^2 = \mathbb{E}\|\bm{Q}\bm{z}\|_p^{2p} = \mathbb{E}(\|\bm{Q}\bm{z}\|_p^p)^2 = \mathbb{E} \left(\sum_{i=1}^n|\bm{q}_i^*\bm{z} |^p  \right)^2 = \sum_{i=1}^n\mathbb{E} |\bm{q}_i\bm{z}|^{2p} + \sum_{i=1}^n \sum_{j=1,j\neq i}^n \mathbb{E} |\bm{q}_i\bm{z}|^p\mathbb{E}|\bm{q}_j\bm{z}|^p \\ \notag
	\leq & Cn^2(1+\eta^2)^p = R_2.
	\end{align}
	
	\textbf{Bound $R_1$}
	By Cauchy inequality, we have
	\begin{align}
	\|\bm{Q}\bar{\bm{z}}\|_p^p \leq \|\bm{Q}\bar{\bm{z}}\|_2^p \leq \|\bm{Q}\|_2^p \|\bar{\bm{z}}\|_2^p = (\|\bar{\bm{z}}\|_2^2)^{p/2} \leq (nB^2)^{p/2} = R_1.
	\end{align}	
	\textbf{Bound $L_f, \bar{L}_f$}
	Note that the sample complexity bound in Lemma \ref{lem:c_st} is propositional to $\log(L_f+\bar{L}_f)$ and $\log(L_f+\bar{L}_f) = \Theta(\log n)$ as long as $L_f$ and $\bar{L}_f$ is in the polynomial of $n$. Thus, it is sufficient to bound them in the polynomials of $n$.
	By calculation, we have
	\begin{align} 
	\bar{L}_f \leq c_1 p n^{p+1} B^p, \quad L_f \leq c_2 n p (1+\eta)^p.
	\end{align}
	
	\textbf{Bound $B_f$}
	
	Applying Lemma \ref{lem:go_st_n}, we have 
	\begin{align} 
	B_f = \mathbb{E}\|\bm{Q}\bm{z}\|_p^p \leq \sum_{i=1}^n\mathbb{E}|\bm{q}_i\bm{z}|^p \leq n\gamma_p(1+\eta^2)^{p/2}.
	\end{align}
	
\end{proof}

\begin{lem} \label{lem:shp_n}
	Suppose $\bm{q} \in \mathbb{S}^{n-1}, r = 0.5\min_{1\leq i \leq n}\{\|\bm{q}-\bm{e}_i\|_2^2,\|\bm{q}+\bm{e}_i\|_2^2\}$, and then we have
	\begin{align} \label{eq:shp_s_n}
	r \leq \frac{C_{\eta,p}}{\theta(1-\theta)}(\theta(1+\eta^2)^{\frac{p}{2}} + (1-\theta)\eta^p-\mathbb{E}_{\Omega}(\|\bm{q}_{\Omega}\|_2^2+\eta^2)^{\frac{p}{2}}),
	\end{align}
	where $C_{\eta,p} = \left((1+\eta^2)^{p/2} + \eta^p - 2(0.5 + \eta^2)^{p/2} \right)^{-1}$.
\end{lem}

\begin{proof}
	
	Without loss of generality, we assume $r = 0.5\min\{\|\bm{q}-\bm{e}_n\|_2^2,\|\bm{q}+\bm{e}_n\|_2^2\}$ and thus $r = |1-q_n|^2 = \epsilon^2$. 
	
	We first show (\ref{eq:shp_s_n}) holds when $\epsilon \leq \frac{1}{\sqrt{2} }$. Let $\Omega' = \Omega\backslash\{n\}$, and we have
	\begin{align*}
	&\mathbb{E}_{\Omega}(\|\bm{q}_{\Omega}\|_2^2 + \eta^2)^\frac{p}{2} =  \theta \mathbb{E}_{\Omega'} (\|\bm{q}_{\Omega'}\|^2+1-\epsilon^2+\eta^2)^{p/2} + (1-\theta )\mathbb{E}_{\Omega'} (\|\bm{q}_{\Omega'}\|^2+\eta^2)^{p/2} \\
	\overset{(a)}{\leq} &\theta(1-\theta)((\epsilon^2+\eta^2)^\frac{p}{2} + (1-\epsilon^2 + \eta^2)^{p/2}) + (1-\theta)^2 \eta^p + \theta^2 (1+\eta^2)^\frac{p}{2}.
	\end{align*}
	Equality in (a) holds only if $\|\bm{q}\|_0 = 2$.
	
	Define $f(\epsilon) = (1+\eta^2)^\frac{p}{2} + \eta^p - (1-\epsilon^2+\eta^2)^\frac{p}{2} - (\epsilon^2 + \eta^2)^\frac{p}{2}$ and $g(\epsilon) = f(\epsilon) - 2f(\frac{1}{\sqrt{2}})\epsilon^2$. For $0 \leq \epsilon \leq \frac{1}{\sqrt{2}}$, we have $g(\epsilon) \geq 0$ and the equality holds only if $\epsilon = 0$ or $\epsilon = \frac{1}{\sqrt{2}}$.
	
	Thus, we have 
	\begin{align*}
	&\inf_{\bm{q} \in \mathbb{S}^{n-1}}\theta(1+\eta^2)^{\frac{p}{2}} + (1-\theta)\eta^p-\mathbb{E}_{\Omega}(\|\bm{q}_{\Omega}\|_2^2+\eta^2)^{\frac{p}{2}} = \theta(1+\eta^2)^{\frac{p}{2}} + (1-\theta)\eta^p- \sup_{\bm{q} \in \mathbb{S}^{n-1}}\mathbb{E}_{\Omega}(\|\bm{q}_{\Omega}\|_2^2+\eta^2)^{\frac{p}{2}}
	\\
	= &\theta(1-\theta)f(\epsilon) \geq 2\theta(1-\theta) f(\frac{1}{\sqrt{2}})\epsilon^2 = 2\theta(1-\theta) f(\frac{1}{\sqrt{2}}) r,
	\end{align*}
	which implies 
	\begin{align*}
	r \leq \frac{C_{\eta,p}}{2\theta(1-\theta)}(\theta(1+\eta^2)^{\frac{p}{2}} + (1-\theta)\eta^p-\mathbb{E}_{\Omega}(\|\bm{q}_{\Omega}\|_2^2+\eta^2)^{\frac{p}{2}}),
	\end{align*}
	for $\epsilon \leq \frac{1}{\sqrt{2}}$, where $C_{\eta,p} = f^{-1}(\frac{1}{\sqrt{2}})$.
	
	Next, we show that (\ref{eq:shp_s_n}) holds when $\epsilon > \frac{1}{\sqrt{2}}$. Considering that we always have $r\leq 1$, we can take the smallest $c$ such that
	\begin{align} \label{eq:ep_g_n}
	1 \leq \frac{c}{2\theta(1-\theta)}(\theta(1+\eta^2)^{\frac{p}{2}} + (1-\theta)\eta^p-\mathbb{E}_{\Omega}(\|\bm{q}_{\Omega}\|_2^2+\eta^2)^{\frac{p}{2}},
	\end{align} 
	holds for $\epsilon > \frac{1}{\sqrt{2}}$.
	
	Denote $\bm{w} = \bm{q}_{1:n-1}$ and define $g(\bm{w}) = \mathbb{E}_{\Omega}(\|[\bm{w},\sqrt{1-\|\bm{w}\|_2^2}]_{\Omega}\|_2^2+\eta^2)^{\frac{p}{2}}$. Then we have
	\begin{align*}
	&\theta(1+\eta^2)^{\frac{p}{2}} + (1-\theta)\eta^p - \mathbb{E}_{\Omega}\|\bm{q}_{\Omega}\|_2^p = \theta(1+\eta^2)^{\frac{p}{2}} + (1-\theta)\eta^p - g(\bm{w}) \\
	\geq& \theta(1+\eta^2)^{\frac{p}{2}} + (1-\theta)\eta^p -  g\left(\frac{\bm{w}}{\sqrt{2}\epsilon}\right) \geq 2\theta(1-\theta)f(\frac{1}{\sqrt{2}})\left(\frac{1}{\sqrt{2}}\right)^2 = f(\frac{1}{\sqrt{2}}) \theta(1-\theta).
	\end{align*}
	
	Thus, we take $c=C_{\eta,p}$ in (\ref{eq:ep_g_n}) to finish the proof.

\end{proof}

\begin{lem} (Sharpness)\label{lem:shp_o_n}
	Suppose $\bm{Q} \in \mathbb{O}(n)$, and then $\exists \bm{P} \in \text{SP}(n)$ such that
	\begin{align}
	\theta(1+\eta^2)^{\frac{p}{2}} + (1-\theta)\eta^p - \frac{1}{nr\gamma_p}  \mathbb{E}_{\bm{Y}} \|\bm{Q}\bm{Y}\|_p^p \geq \frac{\theta(1-\theta)}{2nC_{\eta,p}}\|\bm{Q} - \bm{P}\|_2^2, 
	\end{align}
	where $C_{\eta,p} = \left((1+\eta^2)^{p/2} + \eta^p - 2(0.5 + \eta^2)^{p/2} \right)^{-1}$.
\end{lem}

\begin{proof}
	By a similar argument in Lemma \ref{lem:shp_o}.
\end{proof}

\begin{thm}
	Let $\bm{X} \in \mathbb{R}^{n \times r}, x_{i,j} \sim \mathcal{BG}(\theta)$, $\bm{D}_0 \in \mathbb{O}(n)$ is orthogonal dictionary, and $\bm{Y}_N = \bm{D}_0\bm{X} + \bm{G}$, and $\bm{G} \in \mathbb{R}^{n \times r}$ with $G_{i,j} \sim \mathcal{N}(0,\eta^2)$. Suppose $\hat{\bm{A}}$ is a global maximizer to 
	\begin{equation} 
	\begin{aligned}
	&\underset{\bm{A}}{\text{maximize}}
	& & \|\bm{A}\bm{Y}_N\|_p^p \text{ subject to } \bm{A} \in \mathbb{O}(n), \notag
	\end{aligned}
	\end{equation}
	then for $\delta>0$, there exists a signed permutation $\bm{\Pi}$, such that
	\begin{align*}
	\frac{1}{n}\left\|\hat{\bm{A}}^* - \bm{D}_0\bm{\Pi} \right\|_F^2 \leq C_{\theta} \delta,
	\end{align*}
	with probability at least $1-r^{-1}$ as long as $r=\Omega(\delta^{-2}n\log (n/\delta)((1+\eta^2) n \log n)^{\frac{p}{2}}\xi_\eta^2)$ where $\xi_\eta = (1+\eta^2)^{p/2} + \eta^p - 2(0.5 + \eta^2)^{p/2} $ and $C_{\theta}$ is a constant depends on $\theta$.
\end{thm}

\begin{proof}
	By a similar argument in Theorem \ref{thm:ex_re} and $C_{\theta} = \frac{4}{\gamma_p\theta(1-\theta)}$.
\end{proof}

\section{Convergence Result}
\subsection{Proof of Proposition \ref{prop:snr}} \label{sec:sor}
\begin{prop} \label{prop:snr2}
	Denote $\text{SOR}_i^{(t)}$ as the value of $\text{SOR}_i$ at the $t$-th iteration and $\bm{q} = \bm{a}^{(t)}$ as the variable at the $t$-th iteration. Then the evolution of $\text{SOR}_i$ follows
	\begin{align*}
	\text{SOR}_i^{(t+1)} = \text{SOR}^{(t)}_i \left(1+ \tau_i(\bm{q}) \right),
	\end{align*}
	and 
	\begin{align*}
	\tau_i(\bm{q}) = \frac{\mathbb{E}_{\Omega'} \|[\bm{q}_{\Omega'},q_n]\|_2^k -  \mathbb{E}_{\Omega'} \|[\bm{q}_{\Omega'},q_i]\|_2^k }{ \frac{\theta}{1-\theta} \mathbb{E}_{\Omega'} \|[\bm{q}_{\Omega'},q_i,q_n]\|_2^k +  \mathbb{E}_{\Omega'} \|[\bm{q}_{\Omega'},q_i]\|_2^k },
	\end{align*} 
	where $\Omega' = \Omega\backslash\{n\}\backslash\{i\}$ and $k=p-2$. Two properties of $\tau_i(\bm{q})$ are listed below
	\begin{enumerate}
		\item $0 \leq \tau_i(\bm{q}) \leq \frac{1-\theta}{\theta}$ always holds and $\tau_i(\bm{q}) > 0$ if $q_n > q_i$.
		\item $\tau_i(\bm{q})$ is monotonically increasing in $q_n$ and decreasing in $q_i$. 
	\end{enumerate}
\end{prop}
\begin{proof}
	To compute $\text{SOR}_i$, we separate the population gradient into several parts
	\begin{align*}
	&\mathbb{E}_{\Omega}[\|\bm{a}_{\Omega}\|_2^{k}\bm{a}_{\Omega}]= \mathbb{P}(n \in \Omega, i \in \Omega) \mathbb{E}_{\Omega}[\|\bm{a}_{\Omega}\|_2^{k}\bm{a}_{\Omega}|n \in \Omega, i \in \Omega ] +  \mathbb{P}(n \in \Omega, i \notin \Omega) \mathbb{E}_{\Omega}[\|\bm{a}_{\Omega}\|_2^{k}\bm{a}_{\Omega}|n \in \Omega, i \notin \Omega ]
	\\ + & \mathbb{P}(n \notin \Omega, i \in \Omega) \mathbb{E}_{\Omega}[\|\bm{a}_{\Omega}\|_2^{k}\bm{a}_{\Omega}|n \notin \Omega, i \in \Omega ] 
	+ \mathbb{P}(n \notin \Omega, i \notin \Omega) \mathbb{E}_{\Omega}[\|\bm{a}_{\Omega}\|_2^{k}\bm{a}_{\Omega}|n \notin \Omega, i \notin \Omega ]
	\end{align*}
	where $k = p - 2$.
	
	The probability for each part is 
	\begin{align*}
		&\mathbb{P}(n \in \Omega, i \in \Omega) = \theta^2, 
		&\mathbb{P}(n \in \Omega, i \notin \Omega) = \mathbb{P}(n \notin \Omega, i \in \Omega) = \theta(1-\theta), 
		&&\mathbb{P}(n \notin \Omega, i \notin \Omega) = (1-\theta)^2.
	\end{align*}
	
	Thus, denote $\Omega' = \Omega \backslash \{n\} \backslash \{i\}$ we can the $i$-th 
	\begin{align*}
	&\mathbb{E}_{\Omega}[\|\bm{a}_{\Omega}\|_2^{k}\bm{a}_{\Omega}]_n \\
	= &\theta^2 \mathbb{E}_{\Omega}[\|\bm{a}_{\Omega}\|_2^{k}\bm{a}_{\Omega}|n \in \Omega, i \in \Omega ] +  \theta(1-\theta)\mathbb{E}_{\Omega}[\|\bm{a}_{\Omega}\|_2^{k}\bm{a}_{\Omega}|n \in \Omega, i \notin \Omega ]
 +  \theta(1-\theta) \mathbb{E}_{\Omega}[\|\bm{a}_{\Omega}\|_2^{k}\bm{a}_{\Omega}|n \notin \Omega, i \in \Omega ] \\
	+& (1-\theta)^2 \mathbb{E}_{\Omega}[\|\bm{a}_{\Omega}\|_2^{k}\bm{a}_{\Omega}|n \notin \Omega, i \notin \Omega ] \\ 
	= &\theta^2 \mathbb{E}_{\Omega'}[\|[\bm{a}_{\Omega'}, a_i, a_n]\|_2^{k}] a_i +  \theta(1-\theta) \mathbb{E}_{\Omega'}[\|[\bm{a}_{\Omega'},a_i]\|_2^{k}]a_i + 0 \theta(1-\theta) \mathbb{E}_{\Omega'}[\|[\bm{a}_{\Omega'}, a_n]\|_2^{k}] +  0 (1-\theta)^2 \mathbb{E}_{\Omega'}[\|[\bm{a}_{\Omega'}]\|_2^{k}]  \\
	=& a_i (\theta^2 \mathbb{E}_{\Omega'}[\|[\bm{a}_{\Omega'}, a_i, a_n]\|_2^{k}] +  \theta(1-\theta) \mathbb{E}_{\Omega'}[\|[\bm{a}_{\Omega'},a_i]\|_2^{k}]).
	\end{align*}
	
	Thus, $\text{SOR}_i$ can be computed as 
	\begin{align*}
		\text{SOR}_i^{(t+1)} &= \frac{a_n}{a_i} \frac{(\theta^2 \mathbb{E}_{\Omega'}[\|[\bm{a}_{\Omega'}, a_i, a_n]\|_2^{k}] +  \theta(1-\theta) \mathbb{E}_{\Omega'}[\|[\bm{a}_{\Omega'},a_n]\|_2^{k}]}{(\theta^2 \mathbb{E}_{\Omega'}[\|[\bm{a}_{\Omega'}, a_i, a_n]\|_2^{k}] +  \theta(1-\theta) \mathbb{E}_{\Omega'}[\|[\bm{a}_{\Omega'},a_i]\|_2^{k}]} \\
		&= \frac{a_n}{a_i} \left(1+\frac{\theta(1-\theta) \mathbb{E}_{\Omega'}[\|[\bm{a}_{\Omega'},a_n]\|_2^{k}] - \theta(1-\theta) \mathbb{E}_{\Omega'}[\|[\bm{a}_{\Omega'},a_i]\|_2^{k}]}{\theta^2 \mathbb{E}_{\Omega'}[\|[\bm{a}_{\Omega'}, a_i, a_n]\|_2^{k}] +  \theta(1-\theta) \mathbb{E}_{\Omega'}[\|[\bm{a}_{\Omega'},a_i]\|_2^{k}]}\right) \\
		&= \text{SOR}_i^{(t)} (1+\tau_i(\bm{q})).
	\end{align*}
\end{proof}
\subsection{Proof of Theorem \ref{thm:conv}} \label{sec:conv}
\begin{thm} 
	Assume we apply Algorithm \ref{alg:gpm_lp_s} to solve problem (\ref{eq:lp_s2}) and $\bm{a}^{(0)}$ follows a uniform distribution over the sphere, and denote $\tau(\bm{q}) = \min_{i=1,\cdots,n-1} \tau_i(\bm{q})$, then there exists $T_{\tau} \leq  \log_{1+\tau(\bm{a}^{(0)})}\left(\sqrt{n}\right)$, and $1 \leq i \leq n$ such that
	\begin{align*}
	\left\|\bm{a}^{(t)} - \bm{D}_0\bm{e}_i\right\|_2 \leq \left(1+\tau\left(\bm{a}^{(T_{\tau})}\right)\right)^{T_{\tau} - t}, \forall t \geq T_{\tau}, 
	\end{align*}
	almost surely and the convergence rate $\lim_{k\rightarrow \infty} \frac{\|\bm{a}^{(k+1)} - \bm{D}_0\bm{e}_i\|_2}{\|\bm{a}^{(k)} -  \bm{D}_0\bm{e}_i\|_2} = \frac{1}{\theta}$.
\end{thm}

\begin{proof}
	As discussed in Section \ref{sec:conv}, we assume $\bm{D}_0 = \bm{I}$ and $i=n$. From Proposition \ref{prop:snr2}, we know that for $0<t_1<t_2$
	\begin{align} \label{eq:sor_inc}
		\text{SOR}^{(t_2)} \geq \text{SOR}^{(t_1)} \prod_{i=t_1}^{t_2-1} (1 + \tau(\bm{a}^{(i)})) > \text{SOR}^{(t_1)} (1 + \tau(\bm{a}^{(t_1)}))^{t_2-t_1}
	\end{align}
	
	At initialization, we have
	\begin{align} \label{eq:sor_init}
		\text{SOR}^{(0)} = \frac{a^{(0)}_n}{\|\bm{a}^{(0)}_{-n}\|_2} > \frac{1}{\sqrt{n}}
	\end{align}
	almost surely since we assume $a_n \geq a_i, \forall i$.
	
	Thus, combining (\ref{eq:sor_inc}) and (\ref{eq:sor_init}), there exists $T_\tau \leq \log_{1+\tau(\bm{a}^{(0)})}(\sqrt{n})$ such that $\text{SOR}^{(T_\tau)} > 1$. Then, for $t>T_\tau$, we have
	\begin{align*}
		&\left\|\bm{a}^{(t)}- \bm{e}_n\right\|_2^2 = 2 - 2\sqrt{\frac{(\text{SOR}^{(t)})^2}{(\text{SOR}^{(t)})^2+1}  } \leq \frac{1}{(\text{SOR}^{(t)})^2} \leq \frac{1}{\left(\text{SOR}^{(T_\tau)} (1 + \tau(\bm{a}^{(T_\tau)}))^{t-T_\tau}\right)^2} \\
		\leq &\left(1+\tau\left(\bm{a}^{(T_\tau)}\right)\right)^{2(T_\tau - t)}.
	\end{align*}
\end{proof}

\section{Technical Lemmas}
\begin{lem} \label{lem:gm}
	Suppose $\bm{g} \in \mathbb{R}^n$ with $g_i \sim \mathcal{N}(0,\sigma^2)$ and $\bm{a} \in \mathbb{R}^n$ is a fixed vector. Then 
	$$\mathbb{E}(|\bm{a}^*\bm{g}|^p) = \gamma_p \|\bm{a}\|_2^p,$$
	where $\gamma_p = \sigma^{p} 2^{p/2} \frac{\Gamma(\frac{p+1}{2})}{\sqrt{\pi}}$.
\end{lem}
\begin{proof}
	Due to the rotation invariance property of Gaussian, $\bm{a}^*\bm{g} \sim \mathcal{N}(0,\|\bm{a}\|_2)$. Therefore, $\mathbb{E}(|\bm{a}^*\bm{g}|^p)$ is the $p$-th moment of an absolute Gaussian random variable with zero mean and variance $\|\bm{a}\|_2^2\sigma^2$. By simple calculation $$\mathbb{E}(|\bm{a}^*\bm{g}|^p) = \sigma^{p} 2^{p/2} \frac{\Gamma(\frac{p+1}{2})}{\sqrt{\pi}} \|\bm{a}\|_2^p.$$
\end{proof}

\begin{lem} ($\epsilon$-net over Stiefel manifold) \cite{zhai2019complete} \label{lem:e_st}
	There is a covering $\epsilon$-net $N(\epsilon$) for Stiefel manifold $\mathcal{M} = \{\bm{W} \in \mathbb{R}^{n \times m} \bm{W}^*\bm{W} = \bm{I}, n>m \}$, in operator norm
	\begin{align*}
		\forall \bm{W} \in \mathcal{M}, \exists \bm{W}' \in N(\epsilon), \text{ such that } \|\bm{W}-\bm{W}'\| \leq \epsilon
	\end{align*}
	of size $|N(\epsilon)| \leq \left(\frac{6}{\epsilon}\right)^{nm}$.
	
\end{lem}
\begin{proof}
	See Lemma D.4 in \cite{zhai2019complete}.
\end{proof}

\begin{lem} (\cite{zhang2019structured}) \label{lem:ber}
	Let $v \in \mathbb{R}^d$ with each entry following i.i.d. Ber($\theta$), then
	\begin{align*}
		\mathbb{P}(|\|v\|_0 - \theta d| \geq t\theta d) \leq 2\exp\left(-\frac{3t^2}{2t+6} \theta d  \right)
	\end{align*} 
	
\end{lem}

\begin{proof}
	See Lemma A.4 in \cite{zhang2019structured}.
\end{proof}

\end{document}